\def\year{2021}\relax
\documentclass[letterpaper]{article} 
\usepackage{aaai21}  
\usepackage{times}  
\usepackage{helvet} 
\usepackage{courier}  
\usepackage[hyphens]{url}  
\usepackage{graphicx} 
\usepackage{natbib}  
\usepackage{booktabs}       
\usepackage{amsfonts}       
\usepackage{nicefrac}       
\usepackage{microtype}      
\usepackage[switch]{lineno}
\usepackage{amsmath,amsthm,amssymb}
\usepackage[linesnumbered,ruled,noend]{algorithm2e} 
\usepackage{xcolor}
\usepackage{multirow}
\usepackage{mathtools}     
\usepackage{subfigure}
\usepackage{caption}
\usepackage{textcomp}      

\usepackage{float}
\usepackage{xfrac}
\usepackage{cite}
\usepackage{bm}
\usepackage{paralist}
\usepackage{comment}
\usepackage{stfloats}
\DeclareMathOperator*{\argmin}{arg\,min}
\newtheorem{theorem}{Theorem}[section]

\newcommand{\sol}{u}
\newcommand{\nonlinearOps}{\mathcal N}
\newcommand{\governEqs}{f}

\newcommand{\InitCond}{\sol^{0}({x})}

\newcommand{\SpatialDomain}{\Omega}
\newcommand{\SpatialDim}{d}
\newcommand{\FinalTime}{T}

\newcommand{\solNN}{\tilde{\sol}}
\newcommand{\governEqsNN}{\tilde{\governEqs}}
\newcommand{\NNParams}{\Theta}

\newcommand{\loss}{L}
\newcommand{\lossSol}{\loss_\sol}
\newcommand{\lossEqs}{\loss_\governEqs}
\newcommand{\Number}{N}
\newcommand{\nTrainSol}{\Number_\sol}
\newcommand{\nTrainGovrnEqs}{\Number_\governEqs}

\newcommand{\Ttrain}{\FinalTime_{\text{train}}}
\newcommand{\Tval}{\FinalTime_{\text{val}}}
\newcommand{\Ttest}{\FinalTime_{\text{test}}}

\nocopyright
\author{
    Jungeun Kim\textsuperscript{\rm 1},  
    Kookjin Lee\textsuperscript{\rm 2},
    Dongeun Lee\textsuperscript{\rm 3},
    Sheo Yon Jin\textsuperscript{\rm 4},
    Noseong Park\textsuperscript{\rm 5}
    \\
}
\affiliations{
    \textsuperscript{\rm 1}Department of AI, Yonsei University\\
    \textsuperscript{\rm 2}Extreme Scale Data Science \& Analytics Department, Sandia National Laboratory\\
    \textsuperscript{\rm 3}Department of Computer Science and Information Systems, Texas A\&M University at Commerce\\
    \textsuperscript{\rm 4}IT Engineering Department, Sookmyung Women's University\\
    \textsuperscript{\rm 5}Department of AI \& CS, Yonsei University\\

}

\title{DPM: A Novel Training Method for Physics-Informed Neural Networks in Extrapolation}
\begin{document}
\maketitle

\begin{abstract}
\let\thefootnote\relax\footnotetext{Noseong Park is the corresponding author. This work was supported by the Institute of Information \& Communications Technology Planning \& Evaluation (IITP) grant funded by the Korea government (MSIT) (No. 2020-0-01361, Artificial Intelligence Graduate School Program (Yonsei University)).}
We present a method for learning dynamics of complex physical processes described by time-dependent nonlinear partial differential equations (PDEs). Our particular interest lies in extrapolating solutions in time beyond the range of temporal domain used in training. Our choice for a baseline method is physics-informed neural network (PINN) [Raissi et al., J. Comput. Phys., 378:686--707, 2019] because the method parameterizes not only the solutions, but also the equations that describe the dynamics of physical processes. 
We demonstrate that PINN performs poorly on extrapolation tasks in many benchmark problems. To address this, we propose a novel method for better training PINN and demonstrate that our newly enhanced PINNs can accurately extrapolate solutions in time. Our method shows up to 72\% smaller errors than existing methods in terms of the standard L2-norm metric.
\end{abstract}

\section{Introduction}\label{sec:intro}
Understanding dynamics of complex real-world physical processes is essential in many applications (e.g., fluid dynamics \cite{anderson2016computational,hirsch2007numerical}). Such dynamics are often modeled as time-dependent partial deferential equations (PDEs), where we seek a solution function $\sol(x,t)$ satisfying a governing equation,
\begin{linenomath}\begin{equation}\label{eq:nPDE}
    \governEqs(x,t) \stackrel{\text{def}}{=} \sol_t + \nonlinearOps(\sol) = 0,\, x\in\SpatialDomain,\,t\in[0,\FinalTime],
\end{equation}\end{linenomath}
where $\sol_t \stackrel{\text{def}}{=} \frac{\partial u}{\partial t}$ denotes the partial derivative of $\sol$ w.r.t. $t$, $\nonlinearOps$ denotes a nonlinear differential operator, $\SpatialDomain \subset \mathbb{R}^{\SpatialDim}$ ($\SpatialDim=1,2,3$) denotes a spatial domain, and $\FinalTime$ denotes the final time. Moreover, there are two more types of conditions imposed on the solution function $\sol(x,t)$: i) an initial condition $\sol(x, 0)=u^{0}(x) , \forall x \in \SpatialDomain$ and ii) a set of boundary conditions specifying the behaviors of the solution function on the boundaries of $\SpatialDomain$. Solving such problem becomes particularly challenging when the nonlinear differential operator is highly nonlinear.

Traditionally, classical numerical methods (e.g., \cite{iserles2009first,leveque2002finite,stoer2013introduction}) have been dominant choices for solving such nonlinear time-dependent PDEs, as they have demonstrated their effectiveness in solving complex nonlinear PDEs and provide sound theoretical analyses. However, they are often based on techniques that require complex, problem-specific knowledge such as sophisticated spatio-temporal discretization schemes. 

Recently, with the advancements in deep learning, many data-centric approaches, which heavily rely on the universal approximation theorem \cite{hornik1989multilayer}, have been proposed. 
Most approaches formulate the problem as a rather simple (semi-)supervised learning problem for constructing surrogate models for solution functions \cite{geist2020numerical,khoo2017solving,ling2016reynolds,ling2015evaluation,tripathy2018deep,vlachas2018data,Holl2020Learning}.
Although the formulation itself is simple, this approach requires costly evaluations or existence of solutions 
and also suffers from lack of ways to enforce a priori information of the problem such as physical laws described by the governing equation. There are also more ``physics-aware'' approaches such as methods based on learning latent-dynamics of physical processes \cite{erichson2019physics,fulton2019latent,lee2019deep,lee2020model,wiewel2019latent}. These approaches, however, still require computations of solutions to collect training dataset.

Among those data-centric approaches, a method called physics-informed neural network (PINN) \cite{raissi2019physics} has brought attention to the community because of its simple, but effective way of approximating time-dependent nonlinear PDEs with neural networks, while preserving important physical properties described by the governing equations. PINN achieves these by parameterizing the solution and the governing equation simultaneously with a set of shared network parameters, which we will elaborate in the next section. After the great success of the seminal paper~\cite{raissi2019physics}, many sequels have applied PINN to solve various PDE applications, e.g. \cite{cmc.2019.06641,
Yang2020BPINNsBP,ZHANG2019108850,10.1007/978-3-030-22747-0_15}. 

Nearly all these studies, however, demonstrated the performances of their methods evaluated at a set of testing points sampled within a pre-specified range, i.e., $\{(x^i_{\text{test}},t^i_{\text{test}})\} \subset \SpatialDomain \times [0,\Ttrain] \setminus \{(x^i_{\text{train}},t^i_{\text{train}})\}$, which we denote by \textit{interpolation}. In this paper, however, we are more interested in assessing the capability of PINN as a tool for learning the dynamics of physical processes. In particular, we would like to assess the performance of PINN on a testing set sampled beyond the final training time $\Ttrain$ of the pre-specified range, i.e., $\{(x^i_{\text{test}},t^i_{\text{test}})\} \subset \SpatialDomain \times (\Ttrain, \FinalTime]$, where $\FinalTime > \Ttrain$, and we denote this task by \textit{extrapolation}. In principle, PINN is expected to learn the dynamics Eq.~\eqref{eq:nPDE} and, consequently, to approximate $\sol(x,t)$ in $(\Ttrain, \FinalTime]$ accurately if trained properly. However, in our preliminary study with a one-dimensional viscous Burgers' equation shown in Fig.~\ref{fig:pinn}, we observe that the accuracy of the approximate solution produced by PINN in the extrapolation setting is significantly degraded compared to that produced in the interpolation setting. 

Motivated by this observation, we analyze PINN in detail (Section \ref{sec:pinn}), propose our method to improve the approximation accuracy in extrapolation (Section \ref{sec:dpm}), and demonstrate the effectiveness of the proposed method with various benchmark problems (Section \ref{sec:exps}). In all benchmark problems, our proposed methods, denoted by PINN-D1 and D2, show the best accuracies with various evaluation metrics. In comparison with state-of-the-art methods, errors from our proposed methods are up to 72\% smaller.

\begin{figure}[t]
    \centering
    \subfigure[Interpolation]{\includegraphics[width=0.39\columnwidth]{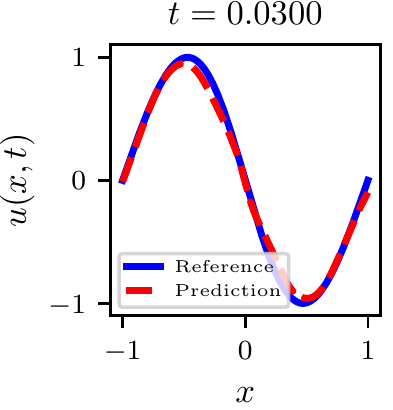}} \hspace{1cm}
    \subfigure[Extrapolation]{\label{fig:pinn_extra}\includegraphics[width=0.39\columnwidth]{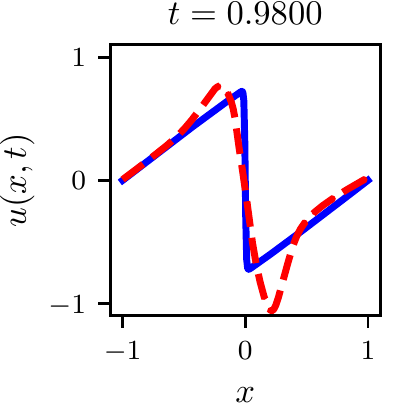}} 
    \caption{1D viscous Burgers' equation examples. We train the PINN model~\cite{raissi2019physics} with $\Ttrain=0.5$ and report two solution snapshots of the reference solution (solid blue line) and the approximated solution (dashed red line) obtained by PINN at $t=0.03$ (i.e., interpolation) and $t=0.98$ (i.e., extrapolation).} 
    \label{fig:pinn}
\end{figure}

\section{Related Work and Preliminaries}\label{sec:pinn}
We now formally introduce PINN. Essentially, PINN parameterizes both the solution $\sol$ and the governing equation $\governEqs$. Let us denote a neural network approximation of the solution $\sol(x,t)$ by $\solNN(x,t;\NNParams)$, where $\NNParams$ denotes a set of network parameters. The governing equation $\governEqs$ is then approximated by a neural network $\governEqsNN(x,t,\solNN;\NNParams) \stackrel{\text{def}}{=} \solNN_t + \nonlinearOps(\solNN(x,t;\NNParams))$, where partial derivatives are obtained via automatic differentiation (or a back-propagation algorithm  \cite{rumelhart1986learning} to be more specific). That is, the neural network  $\governEqsNN(x,t,\solNN;\NNParams)$ shares the same network weights with $\solNN(x,t;\NNParams)$, but enforces physical laws by applying an extra problem-specific nonlinear activation defined by the PDE in Eq.~\eqref{eq:nPDE} (i.e., $\solNN_t + \nonlinearOps(\solNN)$), which leads to the name ``physics-informed'' neural network.\footnote{We also note that there are other studies (e.g.,  \cite{cranmer2020lagrangian,greydanus2019hamiltonian}) using the idea of parameterizing the governing equations, where derivatives are also computed using automatic differentiation.}

This construction suggests that these shared network weights can be learned via forming a loss function consisting of two terms, each of which is associated with approximation errors in $\solNN$ and $\governEqsNN$, respectively. In the original formulation, a loss function consisting of two error terms is considered:
\begin{linenomath}\begin{equation}\label{eq:loss}
    \loss \stackrel{\text{def}}{=} \alpha\loss_\sol + \beta\loss_\governEqs,
\end{equation}\end{linenomath}
where $\alpha,\beta \in \mathbb{R}$ are coefficients and $\loss_\sol, \loss_\governEqs$ are defined below.
\begin{linenomath}\begin{align}
\loss_\sol &= \frac{1}{\nTrainSol}\sum_{i=1}^{\nTrainSol} | \sol(x_\sol^i,t_\sol^i) - \solNN(x_\sol^i, t_\sol^i;\NNParams) |^2,\\ \loss_\governEqs &= \frac{1}{\nTrainGovrnEqs}\sum_{i=1}^{\nTrainGovrnEqs} | \governEqsNN(x_\governEqs^i,t_\governEqs^i,\solNN;\NNParams)|^2.
\end{align}\end{linenomath}

The first loss term, $\loss_\sol$, enforces initial and boundary conditions using a set of training data $\big\{\big((x_\sol^i,t_\sol^i),\sol(x_\sol^i,t_\sol^i)\big)\big\}_{i=1}^{\nTrainSol}$, where the first element of the tuple is the input to the neural network $\solNN$ and the second element is the ground truth that the output of $\solNN$ attempts to match. These data can be easily collected from specified initial and boundary conditions, which are known \textit{a priori} (e.g., $\sol(x,0)=\InitCond=-\sin(\pi x)$ in a PDE we use for our experiments).  The second loss term, $\loss_\governEqs$, minimizes the discrepancy between the governing equation $\governEqs$ and the neural network approximation $\governEqsNN$ evaluated at \textit{collocation points}, which forms another training dataset $\big\{\big((x_\governEqs^i,t_\governEqs^i), \governEqs(x_\governEqs^i,t_\governEqs^i)\big)\big\}_{i=1}^{\nTrainGovrnEqs}$, where the ground truth $\{\governEqs(x_\governEqs^i,t_\governEqs^i)\}_{i=1}^{\nTrainGovrnEqs}$ consists of all zeros.

The advantages of this loss construction are that i) no costly evaluations of the solutions $\sol(x,t)$ at collocation points are required to collect training data, ii) initial and boundary conditions are enforced by the first loss term $\loss_\sol$ where its training dataset can be easily generated, and iii) a physical law described by the governing equation $\governEqs$ in Eq.~\eqref{eq:nPDE} can be enforced by minimizing the second loss term $\loss_\governEqs$. In~\cite{raissi2019physics}, both the loss terms are considered equally important (i.e., $\alpha=\beta=1$), and the combined loss term $\loss$ is minimized.

\paragraph{Motivations.} If PINN can correctly learn a governing equation, its extrapolation should be as good as interpolation. Successful extrapolation will enable the adoption of PINN to many PDE applications. With the loss formulation in Eq.~\eqref{eq:loss}, however, we empirically found that it is challenging to train PINN for extrapolation as shown in Fig.~\ref{fig:pinn}. 

Hence, we first investigate training loss curves of $\lossSol$ and $\lossEqs$ separately: Fig.~\ref{fig:pinn2} depicts the loss curves $\lossSol$ and $\lossEqs$ of PINN trained for a 1D inviscid Burgers' equation. The figure shows that $\lossSol$ converges very fast, whereas $\lossEqs$ starts to fluctuate after a certain epoch and does not decrease below a certain value. From the observation, we can conclude that the initial and the boundary conditions are successfully enforced, whereas the dynamics of the physical process may not be accurately enforced, which, consequently, could lead to significantly less accurate approximations in extrapolation, e.g., Fig.~\ref{fig:pinn_extra}. Motivated by this observation, we propose a novel training method for PINN in the following section. In the experiments section, we demonstrate performances of the proposed training method. 

\begin{figure}[t]
    \subfigure[$\loss_u$ curve]{\includegraphics[height=2.7cm]{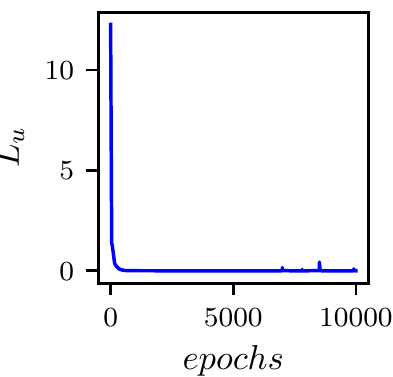}}\hfill
    \subfigure[$\loss_f$ curve]{\label{fig:losseqs}\includegraphics[height=2.7cm]{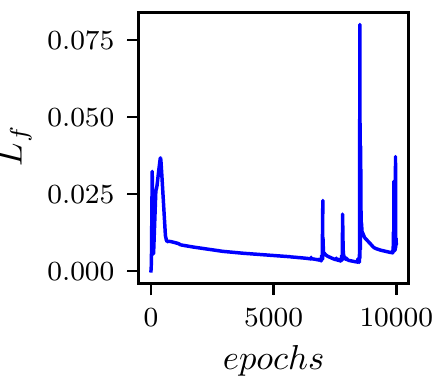}}\hfill
    \subfigure[Updating $\NNParams$ ]{\includegraphics[height=2.8cm]{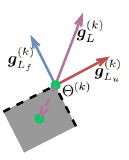}}
    \caption{Example training curves of $\loss_u$ and $\loss_f$ of PINN for a 1D inviscid Burgers' equation in (a) and (b) respectively, and an example of updating $\NNParams$ in (c)}
    \label{fig:pinn2}
\end{figure}

\section{Dynamic Pulling Method (DPM)}\label{sec:dpm}

The issue with training PINN, which we have identified in our preliminary experiments, is that $\lossEqs$ is fluctuating and is not decreasing. To resolve this issue, we propose a novel training method to impose a \emph{soft} constraint of $L_f \leq \epsilon$, where $\epsilon$ is a hyperparameter and can be set to an arbitrary small value to ensure an accurate approximation of the governing equation, i.e., enforcing $\governEqsNN(\cdot)$ to be close to zero. The proposed training concept is dynamically manipulating the gradients.

We dynamically manipulate the gradients of the loss terms on top of a gradient-based optimizer including but not limited to the gradient descent method, i.e., $\NNParams^{(k+1)} = \NNParams^{(k)} - \gamma\bm{g}^{(k)}$, where $\gamma$ is a learning rate, and $\bm{g}^{(k)}$ is a gradient at $k$-th epoch. 
We set the gradient $\bm{g}^{(k)}$ to one of the following vectors depending on conditions:
\begin{linenomath}\begin{align}\label{eq:adaptive}
    \bm{g}^{(k)} = \begin{cases} \bm{g}^{(k)}_{L_u} &\textrm{, if }L_f \leq \epsilon\\
    \bm{g}^{(k)}_{L} &\textrm{, if }L_f > \epsilon \wedge \bm{g}^{(k)}_{L_u} \cdot \bm{g}^{(k)}_{L_f} \geq 0,\\
    \bm{v} + \bm{g}^{(k)}_{L} &\textrm{, otherwise}\end{cases}
\end{align}\end{linenomath}where $\bm{v} \in \mathbb{R}^{\dim(\NNParams)}$ is a manipulation vector, which we will show how to calculate shortly; $\bm{g}^{(k)}_{L_u}$, $\bm{g}^{(k)}_{L_f}$, and $\bm{g}^{(k)}_{L}$ denote the gradients of $L_u$, $L_f$, and $L$, respectively.

Here, we care only about $\bm{g}^{(k)}_{L_u}$, when $L_f$ is small enough, i.e., $L_f \leq \epsilon$, because $L_f$ already satisfies the constraint. There are two possible cases when $L_f > \epsilon$: i) $\bm{g}^{(k)}_{L_u} \cdot \bm{g}^{(k)}_{L_f} \geq 0$ and ii) $\bm{g}^{(k)}_{L_u} \cdot \bm{g}^{(k)}_{L_f} < 0$. In the former case where the two gradient terms $\bm{g}^{(k)}_{L_u}$ and $\bm{g}^{(k)}_{L_f}$ have the same direction (i.e., the angle between them is less than $90^{\circ}$ and hence their dot-product is positive), performing a gradient descent update with $\bm{g}^{(k)}_{L}$ guarantees a decrease in $L_f$. In Fig.~\ref{fig:pinn2} (c), for instance, both $L_f$ and $L_u$ decrease if $\NNParams^{(k)}$ is updated into the gray area.


When $L_f > \epsilon$ and $\bm{g}^{(k)}_{L_u} \cdot \bm{g}^{(k)}_{L_f} < 0$, however, $\bm{v}$ carefully manipulates the gradient in such a way that $L_f$ is guaranteed to decrease after a gradient descent update.

We now seek such a solution $\bm{v}$ that will result in $\big(\bm{v}+\bm{g}^{(k)}_{L}\big) \cdot \bm{g}^{(k)}_{L_f} > 0$ given $\bm{g}^{(k)}_{L}$ and $\bm{g}^{(k)}_{L_f}$. Because the dot-product is distributive, it satisfies the following condition
\begin{linenomath}\begin{align}
\big(\bm{v}+\bm{g}^{(k)}_{L}\big) \cdot \bm{g}^{(k)}_{L_f} = \bm{v}\cdot \bm{g}^{(k)}_{L_f} + \bm{g}^{(k)}_{L} \cdot \bm{g}^{(k)}_{L_f} > 0,
\end{align}\end{linenomath} which can be re-formulated as follows:
\begin{linenomath}\begin{align}\label{eq:prob}
\bm{v}\cdot \bm{g}^{(k)}_{L_f} + \bm{g}^{(k)}_{L} \cdot \bm{g}^{(k)}_{L_f} = \delta,
\end{align}\end{linenomath}where $\delta > 0$ is to control how much we pull $\NNParams^{(k)}$ toward the region where $L_f$ decreases, e.g., the gray region of Fig.~\ref{fig:pinn2} (c).

We note that Eq.~\eqref{eq:prob} has many possible solutions. Among them, one solution, denoted $\bm{v}^* = \frac{-\bm{g}^{(k)}_{L} \cdot \bm{g}^{(k)}_{L_f} + \delta}{\| \bm{g}^{(k)}_{L_f} \|^2_2} \bm{g}^{(k)}_{L_f}$, can be computed by using the \emph{pseudoinverse} of $\bm{g}^{(k)}_{L_f}$, which  is widely used to find such solutions, e.g., the analytic solution of linear least-squared problems arising in linear regressions.

A good characteristic of the pseudoinverse is that it minimizes $\|\bm{v}\|_2^2$~\cite{ben2006generalized}. By minimizing $\|\bm{v}\|_2^2$, we can disturb the original updating process as little as possible. Therefore, we use the pseudoinverse-based solution in our method.

Despite its advantage, the gradient manipulation vector $\bm{v}^*$ sometimes requires many iterations until $L_f \leq \epsilon$. To expedite the pulling procedure, we also dynamically control the additive pulling term $\delta$ as follows:
\begin{linenomath}\begin{align}
    \Delta^{(k)} &= L_f(\NNParams^{(k)}) - \epsilon,\\
    \delta^{(k+1)} &= \begin{cases}w\delta^{(k)},\textrm{ if }\Delta^{(k)} > 0,\\
    \frac{\delta^{(k)}}{w},\textrm{ if }\Delta^{(k)} \leq 0,\end{cases}\label{eq:cnt}
\end{align}\end{linenomath}where $w > 1$ is an inflation factor for increasing $\delta$.

\begin{table*}
\setlength{\tabcolsep}{2pt}
\centering
\scriptsize 
\caption{The extrapolation accuracy in terms of the relative errors in the L2-norm, the explained variance error, the max error, and the mean absolute error in various PDEs. Large (resp. small) values are preferred for $\uparrow$ (resp. $\downarrow$).}\label{tbl:all}
\begin{tabular}{|c|c|c|c|c|c|c|c|c|c|c|c|c|c|c|c|c|}
\hline
\multirow{2}{*}{PDE} & \multicolumn{4}{c|}{L2-norm $(\downarrow)$} & \multicolumn{4}{c|}{Explained variance score $(\uparrow)$} & \multicolumn{4}{c|}{ Max error $(\downarrow)$} & \multicolumn{4}{c|}{ Mean absolute error $(\downarrow)$} \\ \cline{2-17}
          & PINN & PINN-R & PINN-D1 & PINN-D2 & PINN & PINN-R & PINN-D1 & PINN-D2 & PINN & PINN-R & PINN-D1 & PINN-D2 & PINN & PINN-R & PINN-D1 & PINN-D2  \\ \hline
Vis. Burgers & 0.329 & 0.333 & 0.106 & \textbf{ 0.092}  & 0.891  & 0.901 & 0.988 & \textbf{0.991} & 0.657 & 1.081 & 0.545 & \textbf{0.333} & 0.085 &0.108 & 0.026 & \textbf{0.021} \\ \hline
Inv. Burgers & 0.131   & 0.095 & \textbf{0.083} & 0.090  & 0.214 & 0.468 &0.485 & \textbf{0.621}  & 3.088 & 2.589 & \textbf{1.534} & 2.036 & 0.431 & 0.299 & \textbf{0.277} & 0.315 \\ \hline
Allen--Cahn &  0.350 & 0.286 & 0.246 & \textbf{0.182} & 0.090  & 0.919 & 0.939 & \textbf{0.967} & 1.190 & 1.631 & 1.096 & \textbf{0.836} & 0.212 & 0.142 & 0.129 & \textbf{0.094} \\ \hline
Schr\"{o}dinger & 0.239  &0.212  & 0.314  & \textbf{0.141}  & -4.364 & -3.902 & -4.973 & \textbf{-3.257} & 4.656 & 4.222 & 4.945 & \textbf{3.829} & 0.954 & 0.894  & \textbf{0.868} & 0.896 \\ \hline
\end{tabular}
\end{table*}


\begin{figure}[t]
    \centering
    \subfigure[Reference Solution]{\label{fig:visc_pinn_ref}\includegraphics[width=0.23\textwidth]{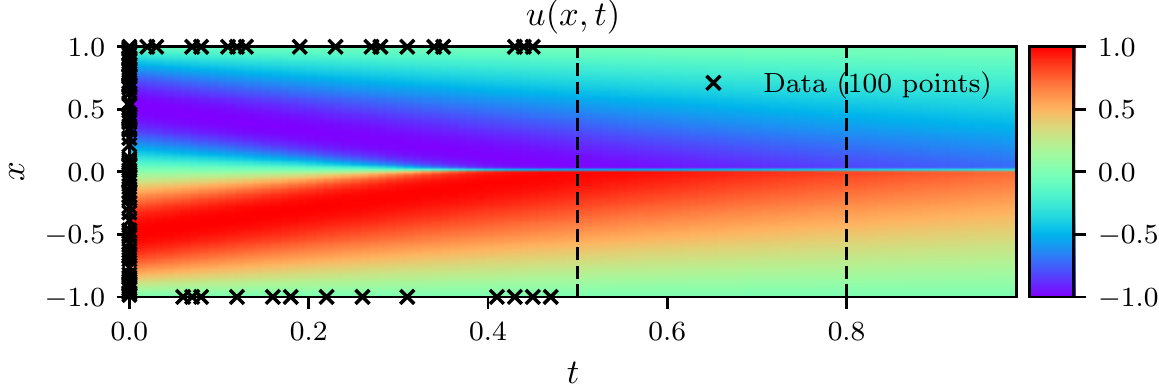}}\hfill
    \subfigure[PINN]{\label{fig:visc_pinn_heatmap}\label{fig:visc_burg_pinn}\includegraphics[width=0.23\textwidth]{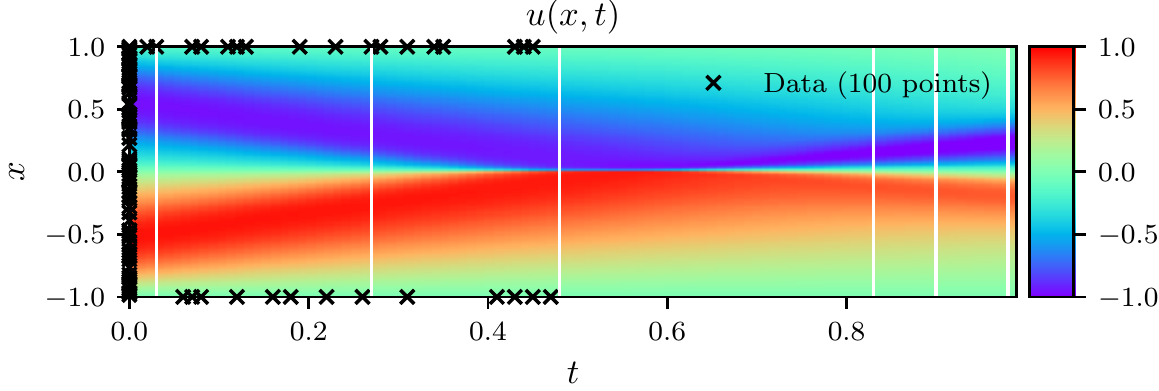}}\\
    \subfigure[PINN-R]{\label{fig:visc_burg_pinnr}\includegraphics[width=0.23\textwidth]{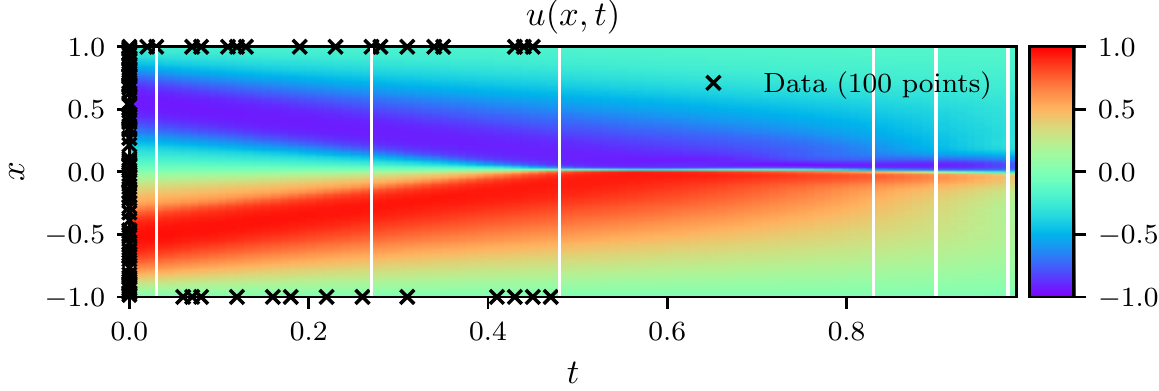}}\hfill
    \subfigure[PINN-D2]{\label{fig:visc_pinndpm_heatmap}\includegraphics[width=0.23\textwidth]{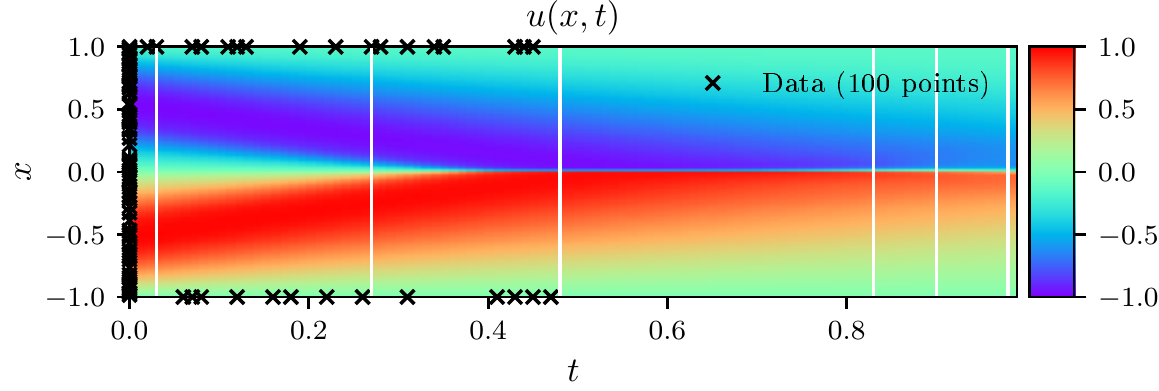}}\\
    \subfigure[PINN]{\label{fig:visc_burg_snap1}\includegraphics[height=2.5cm, width=2.5cm]{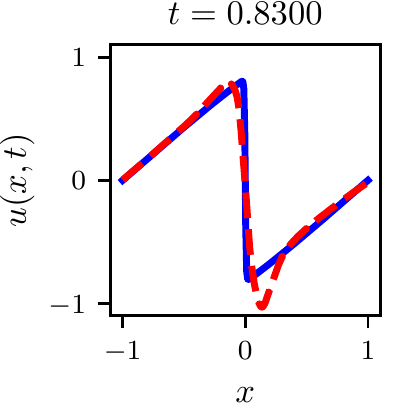}} \hfill
    \subfigure[PINN-R]{\includegraphics[height=2.5cm, width=2.5cm]{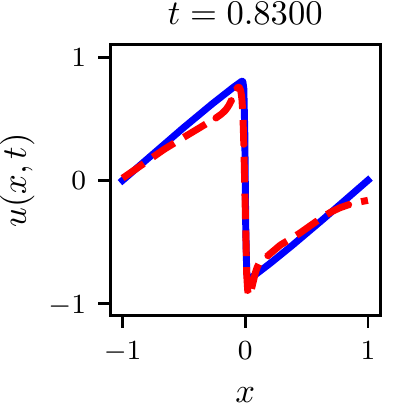}}\hfill
    \subfigure[PINN-D2]{\includegraphics[height=2.5cm, width=2.5cm]{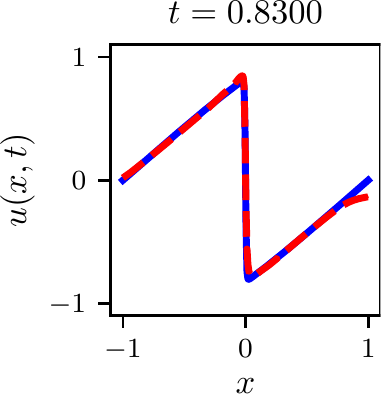}}\\
    \subfigure[PINN]{\includegraphics[height=2.5cm, width=2.5cm]{images/Viscous_1/Burgers-Adam_0.9800_graph_6-40-0.0010-1-7950.pdf}} \hfill
    \subfigure[PINN-R]{\includegraphics[height=2.5cm, width=2.5cm]{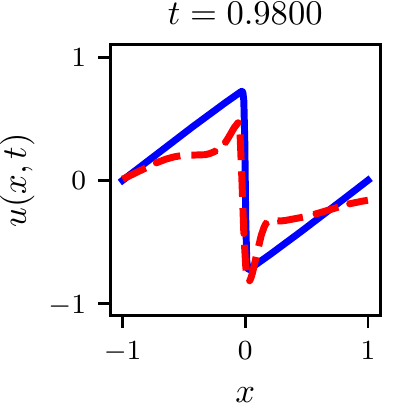}}\hfill
    \subfigure[PINN-D2]{\label{fig:visc_burg_snap6}\includegraphics[height=2.5cm, width=2.5cm]{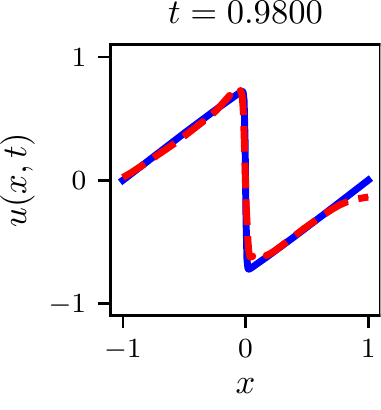}}
    \caption{Top two rows: the complete reference solution and predictions of the benchmark viscous Burgers' equation. The points marked with $\times$ mean initial or boundary points. Bottom: the solution snapshots at $t=\{0.83,0.98\}$ obtained via the extrapolation. In Fig.~\ref{fig:visc_pinn_ref}, the black vertical lines correspond to $\Ttrain$ and $\Tval$, respectively, and in Figs.~\ref{fig:visc_pinn_heatmap}--\ref{fig:visc_pinndpm_heatmap}, the white vertical lines correspond to time indices, where we extract solution snapshots. We refer readers to Appendix for more solution snapshots. The meanings of the vertical lines remain the same in the following figures.} 
    \label{fig:burgers}
\end{figure}

\section{Experiments}\label{sec:exps}
We describe our experimental environments and results with four benchmark time-dependent nonlinear PDEs and several different neural network designs. Our software and hardware environments are as follows: \textsc{Ubuntu} 18.04 LTS, \textsc{Python} 3.6.6, \textsc{Numpy} 1.18.5, \textsc{Scipy} 1.5, \textsc{Matplotlib} 3.3.1, \textsc{TensorFlow-gpu} 1.14, \textsc{CUDA} 10.0, and \textsc{NVIDIA} Driver 417.22, i9 CPU, and \textsc{NVIDIA RTX Titan}.

\subsection{Experimental Environments}
\paragraph{PDEs.} We consider viscous and inviscid Burgers' equations, nonlinear Schr\"{o}dinger equation (NLS), and Allen--Cahn (AC) equation. We refer readers to our supplementary material for detailed descriptions for these equations.

For training/validating/testing, we divide the entire time domain $[0,T]$ into three segments: $[0,\Ttrain]$, $(\Ttrain, \Tval]$, and $(\Tval, \Ttest]$, where $T=\Ttest > \Tval > \Ttrain > 0$. In other words, our task is to predict the solution functions of the PDEs in a future time frame, i.e., extrapolation. We use $\Ttrain = \frac{T}{2}$, $\Tval = \frac{4T}{5}$, and $\Ttest = T$, i.e., extrapolating for the last 20\% of the time domain, which is a widely used setting in many time-series prediction studies~\cite{KIM2003307,10.1145/2835776.2835834}.

\paragraph{Baselines.} Our task definition is not to simply approximate a solution function $u$ with a regression model but to let a neural network learn physical dynamics without costly collections of training samples (see our broader impact statement to learn why it is important to train without costly collections of training samples). For this task, the  state-of-the-art method is PINN. We compare our method with the following baselines: i) the original PINN which uses a series of fully-connected and hyperbolic tangent layer, denoted by PINN, and ii) PINN improved with the residual connection~\cite{DBLP:conf/cvpr/HeZRS16}, denoted by PINN-R. We apply our DPM with (resp. without) controlling $\delta$ in Eq.~\eqref{eq:cnt} to train PINN-R, denoted by PINN-D2 (resp. PINN-D1).

\paragraph{Evaluation Metrics.} For performance evaluation, we collect predicted solutions  at testing data instances to construct a solution vector $\bm{\solNN} = [\solNN(x^1_{\text{test}},t^1_{\text{test}};\NNParams),\solNN(x^2_{\text{test}},t^2_{\text{test}};\NNParams),\ldots]^{\top}$, where $\{(x^i_{\text{test}},t^i_{\text{test}})\}$ is a set of testing samples. $x^i_{\text{test}}$ is sampled at a uniform spatial mesh grid in $\SpatialDomain$ and $t^i_{\text{test}}$ is on a uniform temporal grid in $(\Tval, \Ttest]$. See Appendix for how to build testing sets. For the comparison, we also collect the reference solution vector, denoted $\bm{\sol}$, at the same testing data instances by solving the same PDEs using traditional numerical solvers.
As evaluation metrics, we use the standard relative errors in L2-norm, i.e., ${\|\bm{\solNN}-\bm{\sol}\|_2}/{\|\bm{\sol}\|_2}$, the explained variance score, the max error, and the mean absolute error, each of which shows a different aspect of performance. Moreover, we report snapshots of the reference and approximate solutions at certain time indices.


\paragraph{Hyperparameters.} For all methods, we test with the following hyperparameter configurations: the number of layers is $\{2,3,4,5,6,7,8\}$, the dimensionality of hidden vector is $\{20, 40, 50, 100, 150\}$. For PINN and PINN-R, we use $\alpha=\{1, 10, 100, 1000\}$, $\beta=\{1, 10, 100, 1000\}$ --- we do not test the condition of $\alpha=\beta$, except for $\alpha=\beta=1$. Our DPM uses $\alpha = \beta = 1$. The learning rate is $\{\text{1e-3, 5e-3, 1e-4, 5e-5}\}$ with various standard optimizers such as Adam, SGD, etc. For the proposed DPM, we test with $\epsilon=\{0.001, 0.005, 0.01, 0.0125\}$, $\delta=\{0.01, 0.1, 1, 10\}$, and $w=\{1.001, 1.005, 1.01, 1.025\}$. We also use the early stopping technique using the validation error as a criterion. If there are no improvements in validation loss larger than 1e-5 for the past 50 epochs, we stop the training process. We choose the model that best performs on the validation set.


\paragraph{Train \& Test Set Creation.}
To build testing sets, $x^i_{\text{test}}$ is sampled at a uniform spatial mesh grid in $\SpatialDomain$ and $t^i_{\text{test}}$ is on a uniform temporal grid in $(\Tval, \Ttest]$. We use a temporal step size of $0.01$, $0.0175$, $\frac{0.01\pi}{2}$, and $0.005$ for the viscous Burgers' equation, the inviscid Burgers' equation, the NLS equation, and the AC equation, respectively. We divide $\SpatialDomain$ into a grid of 256, 512, 256, and 256 points for the viscous Burgers' equation, the inviscid Burgers' equation, the NLS equation, and the AC equation, respectively.

For creating our training sets, we use $\nTrainSol=100$ initial and boundary tuples for all the benchmark equations. For $\nTrainGovrnEqs$, we use 10K for the viscous and the inviscid Burgers' equations, and 20K for the NLS equation and the AC equation.

\subsection{Experimental Results}

Table~\ref{tbl:all} summarizes the overall performance for all benchmark PDEs obtained by PINN, PINN-R, PINN-D1, and PINN-D2. PINN-R shows smaller L2-norm errors than PINN. The proposed PINN-D2 significantly outperforms PINN and PINN-R in all four benchmark problems for all metrics. For the viscous Burgers' equation and the AC equation, PINN-D2 demonstrates 72\% and 48\% (resp. 72\% and 36\%) improvements over PINN (resp. PINN-R) in terms of the relative L2-norm, respectively.

\paragraph{Viscous Burgers' equation.} Fig.~\ref{fig:burgers} shows the reference solution and  
predictions made by PINN and the PINN variants of the viscous Burgers' equation. In Figs.~\ref{fig:visc_burg_pinn}--\ref{fig:visc_burg_pinnr}, both PINN and PINN-R fail to correctly learn the governing equation and their prediction accuracy is significantly degraded as $t$ increases. However, the proposed PINN-D2 shows much more accurate prediction even when $t$ is close to the end of the time domain. These results 
explain that learning a governing equation correctly helps accurate extrapolation. 
Although PINN and PINN-R are able to learn the initial and boundary conditions accurately, their extrapolation performances are poor because they fail to learn the governing equation accurately. 
Figs.~\ref{fig:visc_burg_snap1}--\ref{fig:visc_burg_snap6} report solution snapshots at $t=\{0.83,0.98\}$ and we observe that the proposed PINN-D2 outperforms the other two PINN methods. Only PINN-D2 accurately enforces the prediction around $x=0$ in Fig.~\ref{fig:visc_burg_snap6}. PINN-D1 is comparable to PINN-D2 in this equation according to Table~\ref{tbl:all}.

\begin{figure}[t]
    \centering
    \subfigure[Reference Solution]{\includegraphics[width=0.23\textwidth]{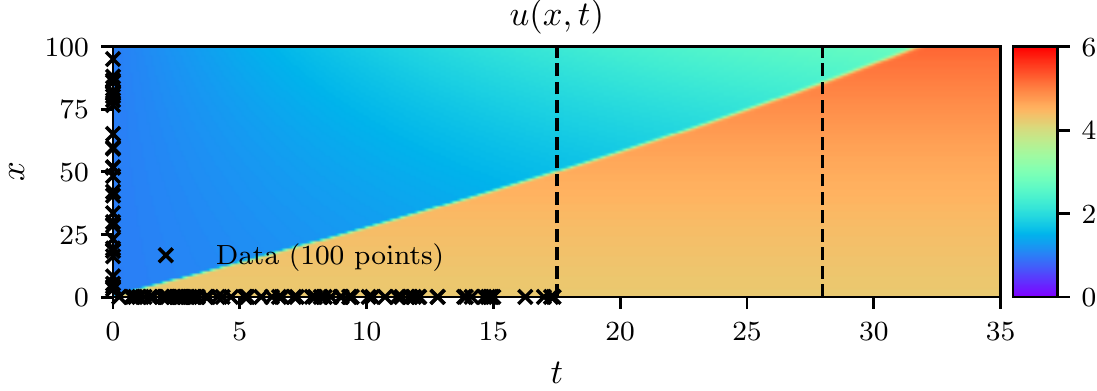}} \hfill
    \subfigure[PINN]{\includegraphics[width=0.23\textwidth]{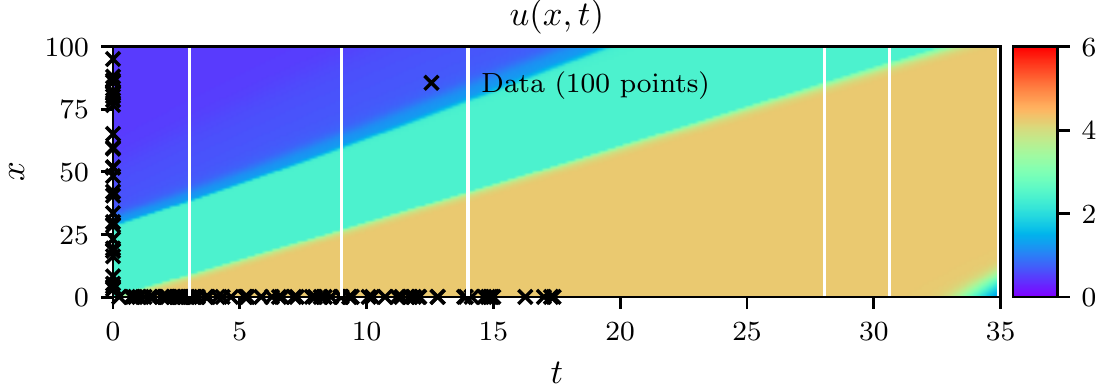}}\\
    \subfigure[PINN-R]{\includegraphics[width=0.23\textwidth]{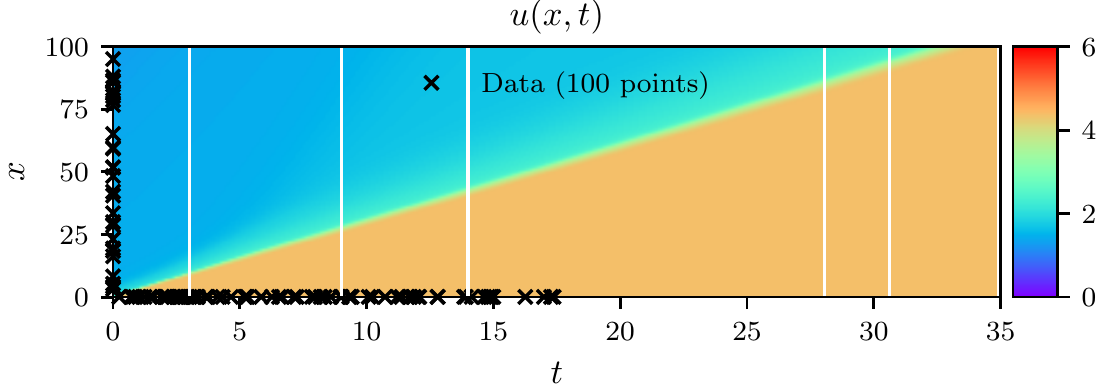}}\hfill
    \subfigure[PINN-D2]{\includegraphics[width=0.23\textwidth]{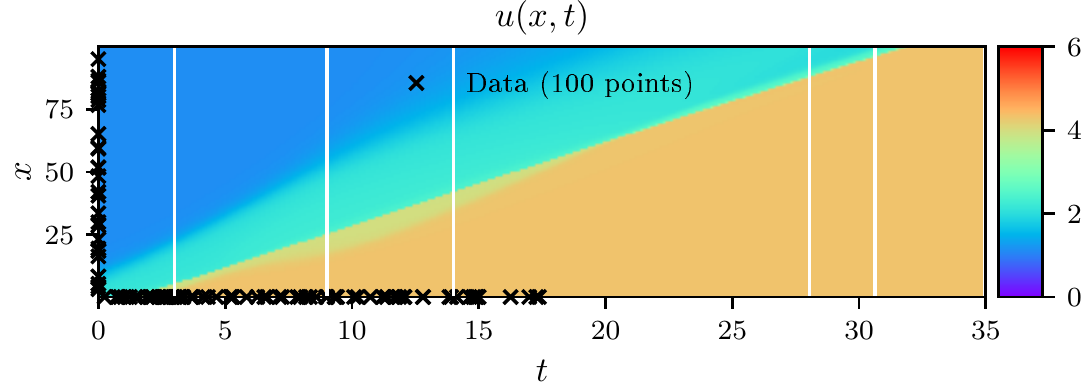}}\\
    \subfigure[PINN]{\includegraphics[height=2.5cm, width=2.5cm]{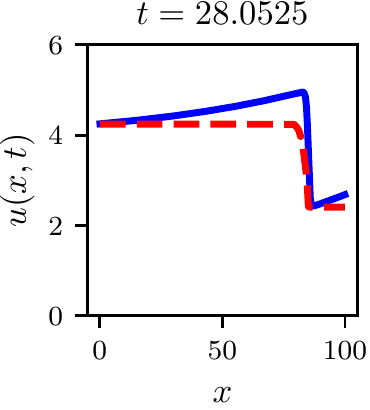}} \hfill
    \subfigure[PINN-R]{\includegraphics[height=2.5cm, width=2.5cm]{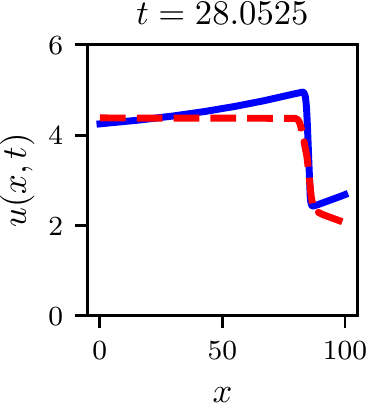}}\hfill
    \subfigure[PINN-D2]{\label{fig:invisc_burg_snap1}\includegraphics[height=2.5cm, width=2.5cm]{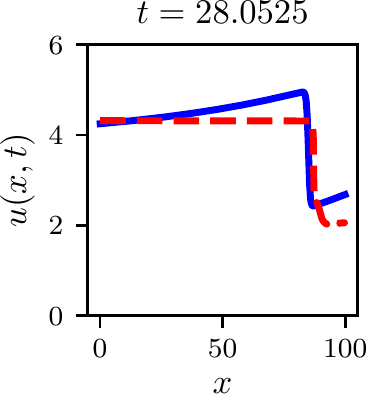}}\hfill
    \subfigure[PINN]{\includegraphics[height=2.5cm, width=2.5cm]{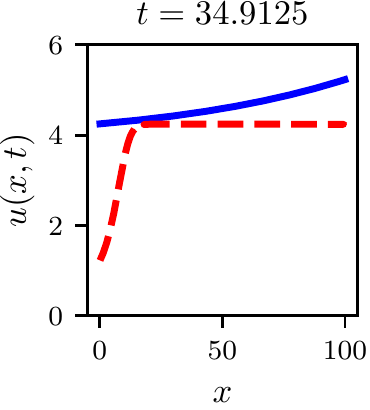}} \hfill
    \subfigure[PINN-R]{\includegraphics[height=2.5cm, width=2.5cm]{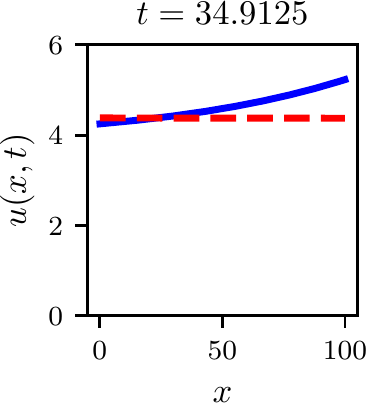}}\hfill
    \subfigure[PINN-D2]{\label{fig:invisc_burg_snap2}\includegraphics[height=2.5cm, width=2.5cm]{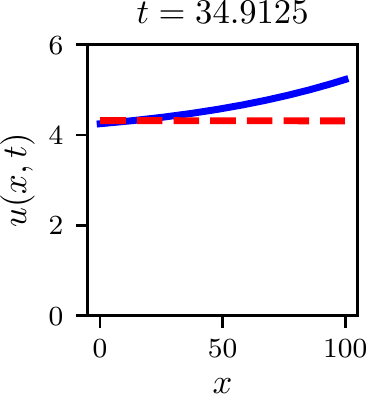}}
    \caption{Top two rows: the complete reference solution and predictions of the benchmark inviscid Burgers' equation. The points marked with $\times$ mean initial or boundary points. Bottom: the solution snapshots at $t=\{28.0875,34.9125\}$ obtained via the extrapolation.} 
    \label{fig:burgers2}
\end{figure}

\paragraph{Inviscid Burgers' equation.} In this benchmark problem, we consider the inviscid Burgers' equation posed on a very long time domain $[0, 35]$, which is much larger than those of other benchmark problems and could make the extrapolation task even more challeging. Fig.~\ref{fig:burgers2} reports the results obtained by the PINN variants along with the reference solution. All the three methods, PINN-R, PINN-D1, and PINN-D2, are comparable in this benchmark problem. However, we can still observe that PINN-D2 produces slightly more accurate predictions than other methods at $x=0$, the boundary condition. The first condition of Eq.~\eqref{eq:adaptive} accounts for this result: when $\lossEqs$ is sufficiently small, the update performed by DPM further decreases $\lossSol$ to improve the predictions in the initial and boundary conditions.

\begin{figure}[t]
    \centering
    \subfigure[Reference Solution]{\includegraphics[width=0.23\textwidth]{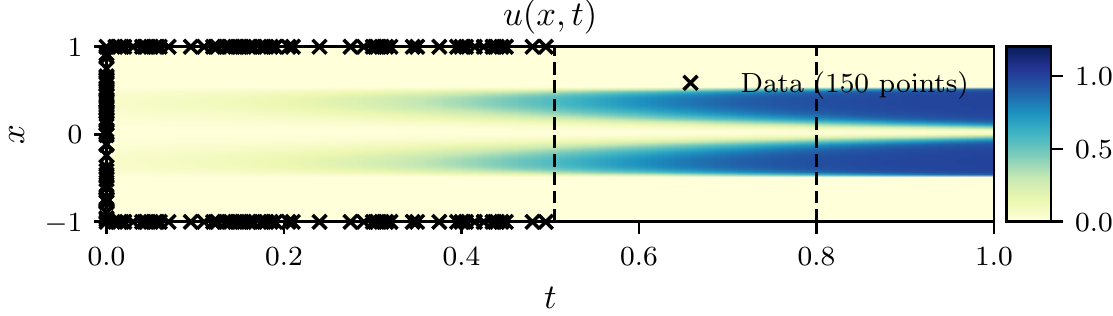}} \hfill
    \subfigure[PINN]{\includegraphics[width=0.23\textwidth]{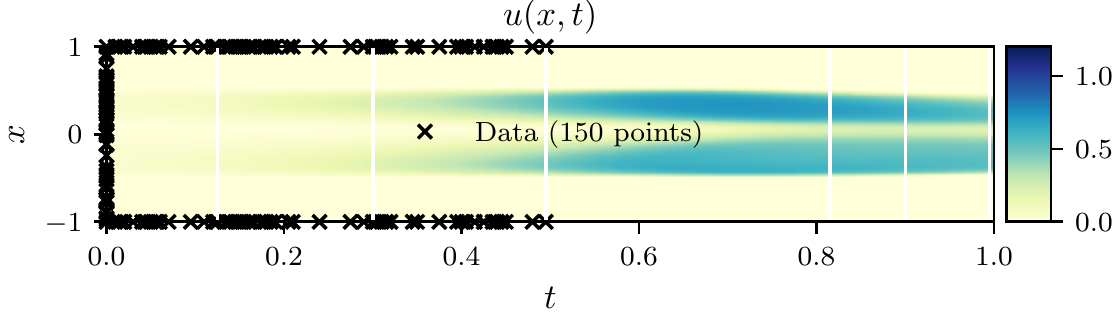}}\\
    \subfigure[PINN-R]{\includegraphics[width=0.23\textwidth]{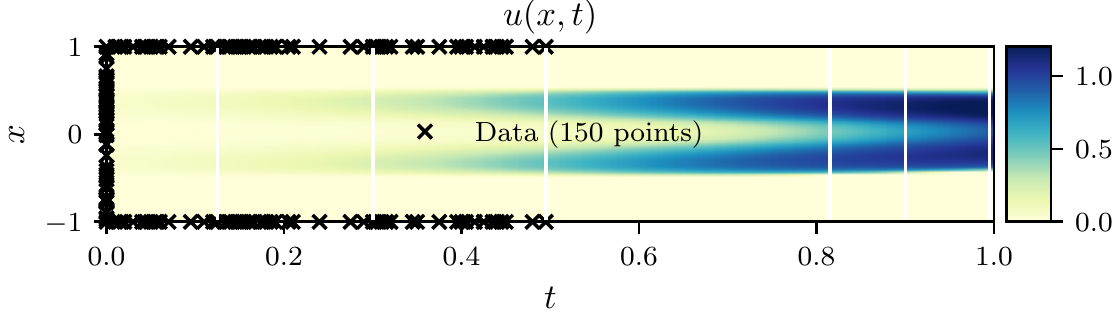}}\hfill
    \subfigure[PINN-D2]{\includegraphics[width=0.23\textwidth]{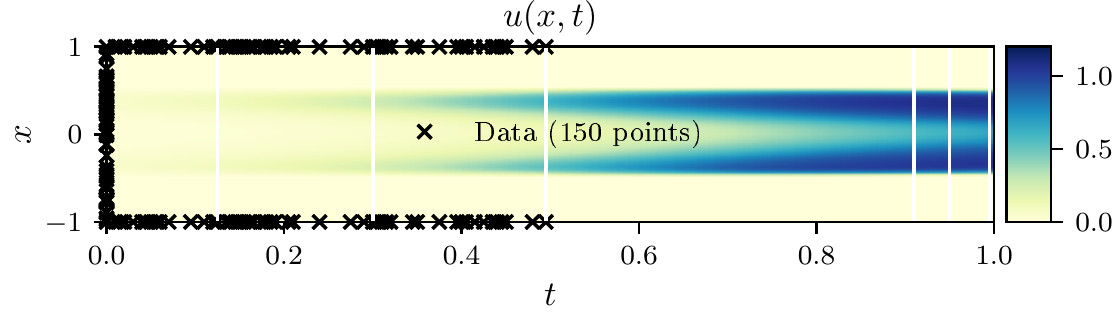}}\\
    \subfigure[PINN]{\label{fig:ac_snap1}\includegraphics[height=2.5cm, width=2.5cm]{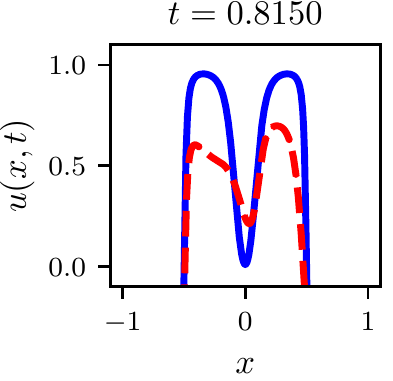}} \hfill
    \subfigure[PINN-R]{\includegraphics[height=2.5cm, width=2.5cm]{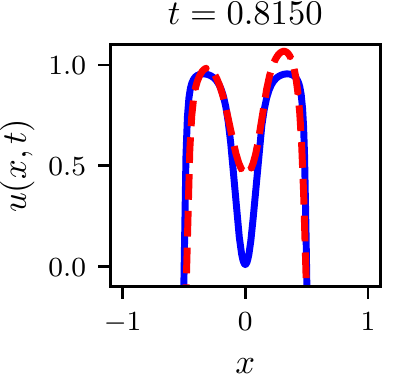}}\hfill
    \subfigure[PINN-D2]{\includegraphics[height=2.5cm, width=2.5cm]{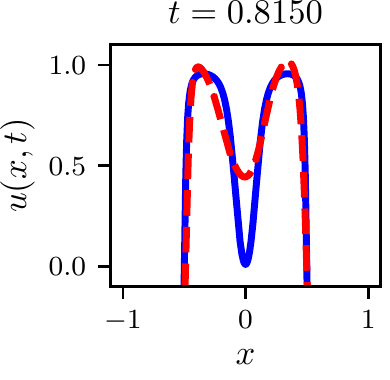}}\hfill
    \subfigure[PINN]{\includegraphics[height=2.5cm, width=2.5cm]{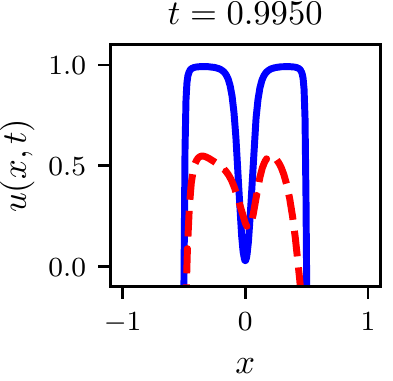}} \hfill
    \subfigure[PINN-R]{\includegraphics[height=2.5cm, width=2.5cm]{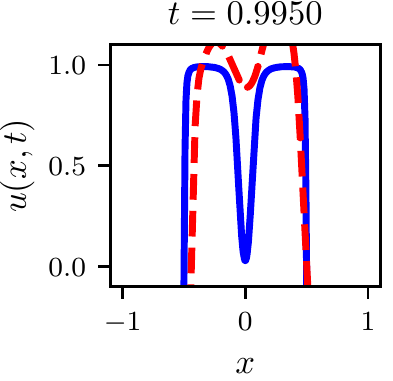}}\hfill
    \subfigure[PINN-D2]{\label{fig:ac_snap6}\includegraphics[height=2.5cm, width=2.5cm]{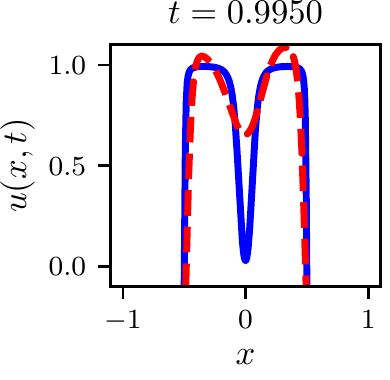}}
    \caption{Top two rows: the complete reference solution and predictions of the Allen--Cahn equation. The points marked with $\times$ mean initial or boundary points. Bottom: the extrapolation solution snapshots at  $t=\{0.815,0.995\}$.} 
    \label{fig:ac}
\end{figure}

\paragraph{Allen--Cahn equation (AC).} Fig.~\ref{fig:ac} reports the reference solutions of the AC equation and the predictions made by all the considered PINN variants. The solution snapshots shown in Figs.~\ref{fig:ac_snap1}--\ref{fig:ac_snap6} demonstrate that the proposed PNN-D2 produces the most accurate approximations to the reference solutions. In particular, the approximate solutions obtained by using PINN-D2 matches very closely with the reference solutions with the exception on the valley (around $x=0$), where all three methods struggle to make accurate predictions. Moreover, the approximate solutions of PINN-D2 are almost symmetric w.r.t. $x=0$, whereas the approximate solutions of the other two methods are significantly non-symmetric and the accuracy becomes even more degraded as $t$ increases.


\begin{figure}[t]
    \centering
    \subfigure[Reference Solution]{\includegraphics[width=0.23\textwidth]{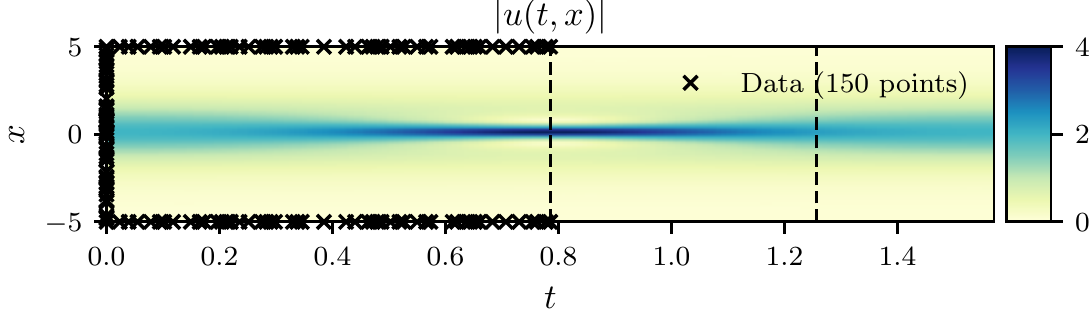}} \hfill
    \subfigure[PINN]{\includegraphics[width=0.23\textwidth]{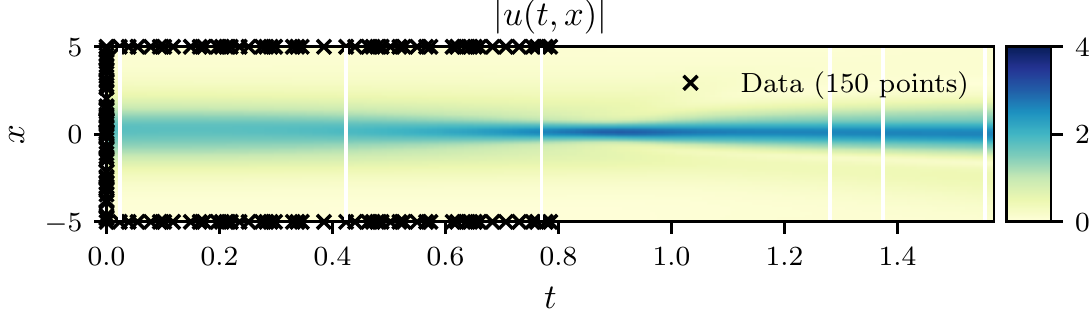}}\\
    \subfigure[PINN-R]{\includegraphics[width=0.23\textwidth]{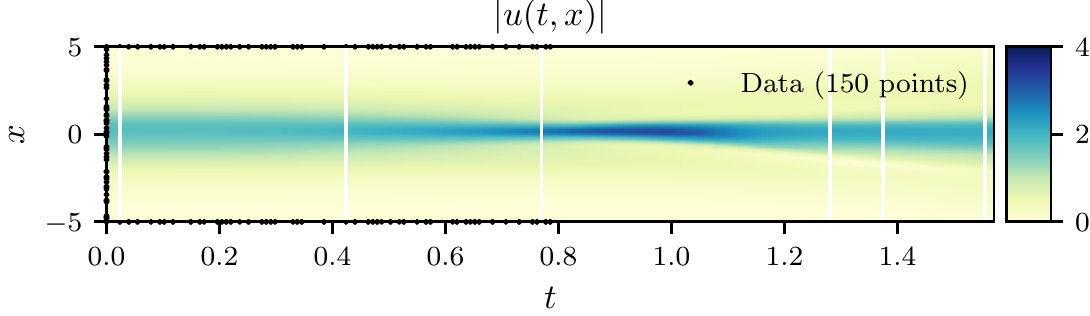}}\hfill
    \subfigure[PINN-D2]{\includegraphics[width=0.23\textwidth]{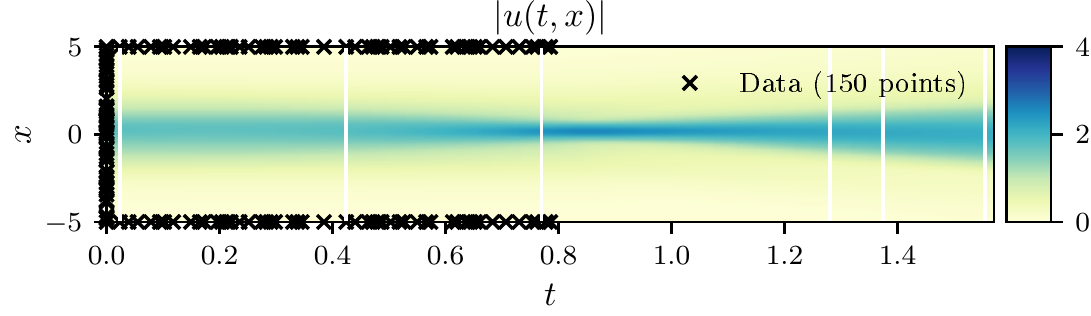}}\\
    \subfigure[PINN]{\includegraphics[height=2.5cm, width=2.5cm]{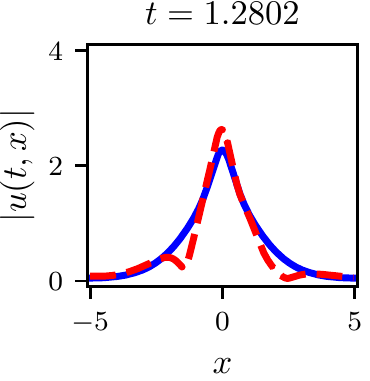}} \hfill
    \subfigure[PINN-R]{\includegraphics[height=2.5cm, width=2.5cm]{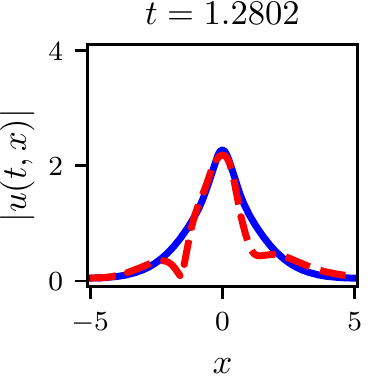}}\hfill
    \subfigure[PINN-D2]{\includegraphics[height=2.5cm, width=2.5cm]{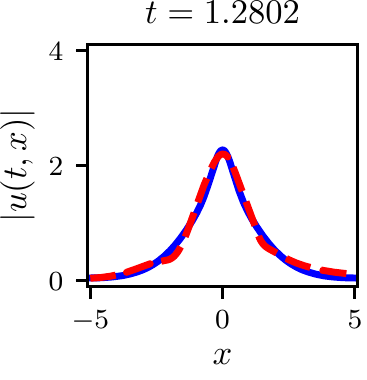}}\hfill
    \subfigure[PINN]{\includegraphics[height=2.5cm, width=2.5cm]{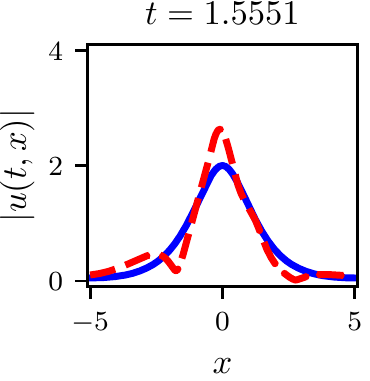}} \hfill
    \subfigure[PINN-R]{\includegraphics[height=2.5cm, width=2.5cm]{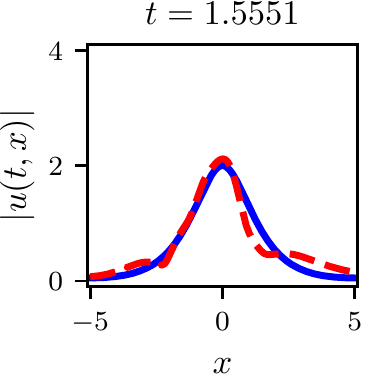}}\hfill
    \subfigure[PINN-D2]{\label{fig:sch_dpm}\includegraphics[height=2.5cm, width=2.5cm]{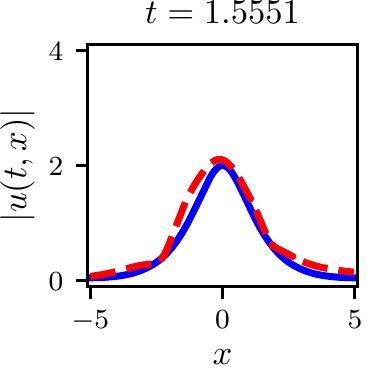}}
    \caption{Top two rows: the complete reference solution and predictions of the nonlinear Schr\"{o}dinger equation. The points marked with $\times$ mean initial or boundary points. Bottom: the extrapolation solution snapshots at $t=\{1.2802,1.5551\}$.} 
    \label{fig:sch}
\end{figure}

\paragraph{Nonlinear Schr\"{o}dinger equation (NLS).} 
Fig.~\ref{fig:sch} reports the reference solution of the NLS equation and the predictions made by all the considered PINN variants. Because the solution of the NLS equation is a complex-valued, the magnitudes of the reference solution $|\sol(x,t)|$ and the predictions $|\solNN(x,t)|$ are depicted. The solution snapshots produced by PINN and PINN-R exhibit errors around $x=-1$ and $x=1$ whereas PINN-D2 is accurate around the region. In particular, the predictions made by PINN and PINN-R exhibit the shapes that are very similar to previous time steps' solution snapshots, which indicates that the dynamics of the system is not learned accurately. In contrast, PINN-D2 seems to enforce the dynamics much better and produce more accurate predictions.

\subsection{Ablation Study}
To show the efficacy of controlling $\delta$ in Eq.~\eqref{eq:cnt}, we compare PINN-D1 and PINN-D2. In Table~\ref{tbl:all}, PINN-D2 outperforms PINN-D1 for three benchmark equations. The biggest improvement is made in the NLS equation, one of the most difficult equations to predict, i.e., 0.314 vs. 0.141 in the L2-norm metric. We note that without controlling $\delta$, PINN-D1 shows worse predictions even than PINN and PINN-R in this equation.

\subsection{Visualization of Training Process}
Fig.~\ref{fig:pinn3} shows the curves of $L_u$ and $L_f$ with our method in the benchmark viscous Burgers' equation. For $L_f$, we set $\epsilon=0.001$, $\delta=0.01$, and $w=1.01$, which produces the best extrapolation accuracy. With this setting, DPM immediately pulls $\lossEqs$ toward the threshold $\epsilon=0.001$ as soon as $L_f > 0.001$. Because our method uses the smallest manipulation vector, $\bm{v}^*$, $L_u$ is also trained properly as training goes on.

\begin{figure}[t]
\centering
\subfigure[$\loss_u$ curve]{\includegraphics[height=3cm]{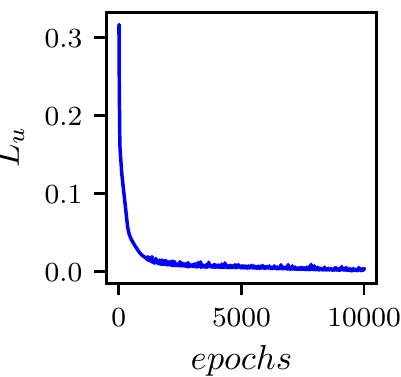}}\hspace{1cm}
\subfigure[$\loss_f$ curve]{\label{fig:losseqs2}\includegraphics[height=3cm]{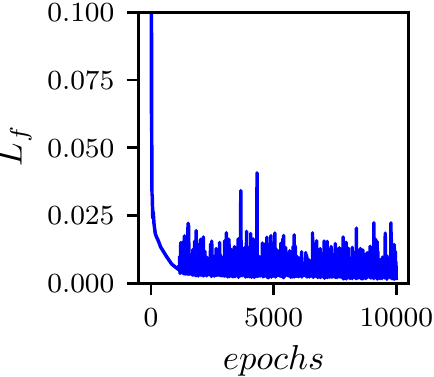}}
\caption{Example training curves of $\loss_u$ and $\loss_f$ of PINN-D2 ($\epsilon=0.001, \delta=0.01$, and $w=1.01$) for the benchmark viscous Burgers' equation}\label{fig:pinn3}
\end{figure}

\subsection{PINN vs. Regression}
The task definition of PINN is different from that of the simple regression learning a solution function $u$, where the reference solutions of $u(x,t)$ are collected not only for initial and boundary conditions but also for other $(x,t)$ pairs. In general, this approach requires non-trivial efforts to run computer simulations and collect such reference solutions. Once we collect them, one advantage is learning $u$ becomes a 
simple regression task without involving $L_f$. However, a critical disadvantage is that governing equations cannot be explicitly imposed during the training process.

Although our task is not to fit a regression model to the reference solutions but to learn physical dynamics, we compare our proposed method with the regression-based approach to better understand our method. To train the regression model, we use $L_u$ with an augmented training set $\big\{\big((x_\sol^i,t_\sol^i),\sol(x_\sol^i,t_\sol^i)\big)\big\}_{i=1}^{\nTrainSol} \cup \big\{\big((x_r^i,t_r^i),\sol(x_r^i,t_r^i)\big)\big\}_{i=1}^{N_r}$, where the first set consists of initial and boundary training samples, $(x_r^i,t_r^i)$ are sampled uniformly in $\SpatialDomain$ and $[0, \Ttrain]$, and we set $N_r=\nTrainGovrnEqs$ for fairness. 
We run external software to calculate $\sol(x_r^i,t_r^i)$, which is not needed for $\sol(x_\sol^i,t_\sol^i)$ because initial and boundary conditions are known a priori.

\begin{figure}[!t]
    \centering
    \subfigure[FC]{\includegraphics[height=2.5cm, width=2.5cm]{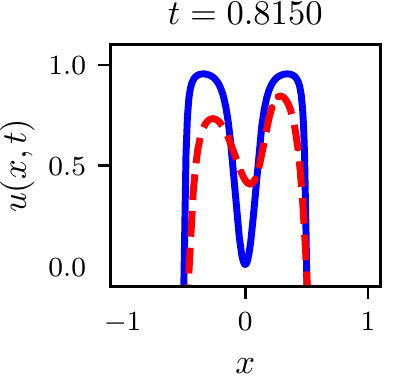}} \hfill
    \subfigure[FC]{\includegraphics[height=2.5cm, width=2.5cm]{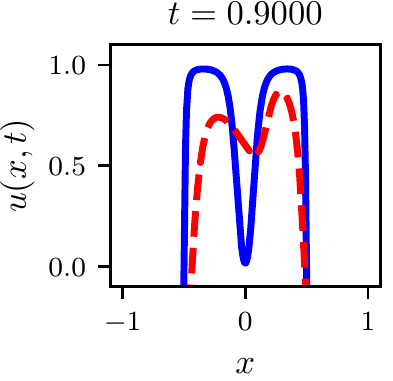}}\hfill
    \subfigure[FC]{\includegraphics[height=2.5cm, width=2.5cm]{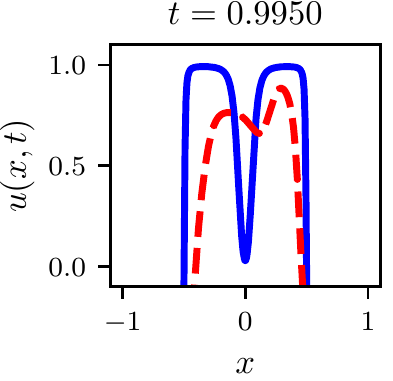}}\hfill
    \subfigure[FC-R]{\includegraphics[height=2.5cm, width=2.5cm]{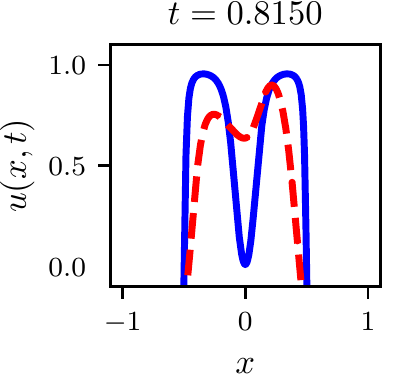}}\hfill 
    \subfigure[FC-R]{\includegraphics[height=2.5cm, width=2.5cm]{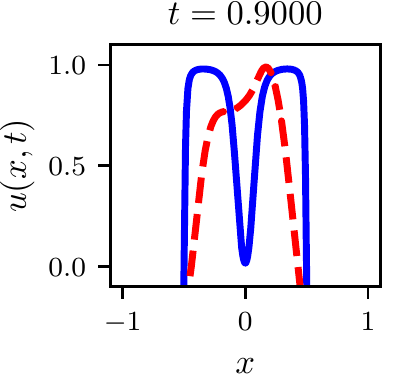}}\hfill
    \subfigure[FC-R]{\label{fig:reg_ac}\includegraphics[height=2.5cm, width=2.5cm]{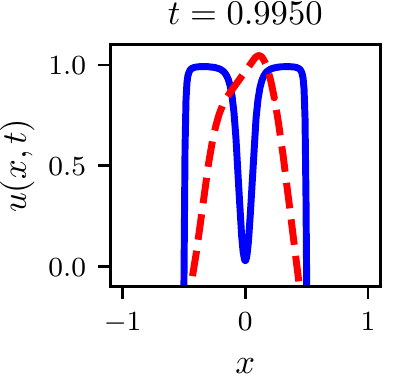}}\hfill
    \subfigure[FC]{\includegraphics[height=2.5cm, width=2.5cm]{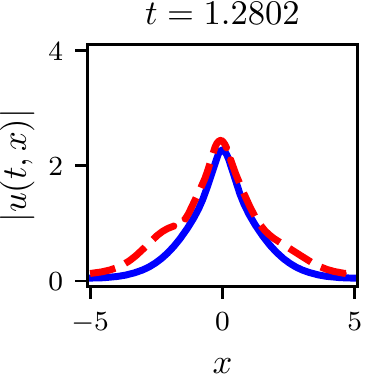}} \hfill
    \subfigure[FC]{\includegraphics[height=2.5cm, width=2.5cm]{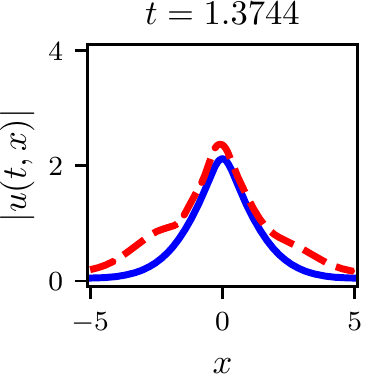}}\hfill
    \subfigure[FC]{\includegraphics[height=2.5cm, width=2.5cm]{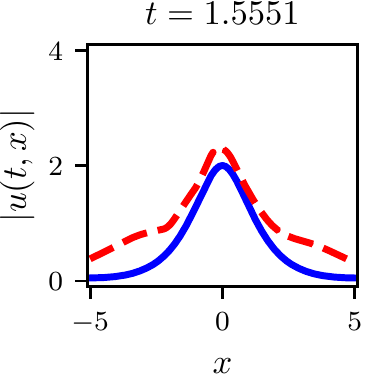}}\hfill
    \subfigure[FC-R]{\includegraphics[height=2.5cm, width=2.5cm]{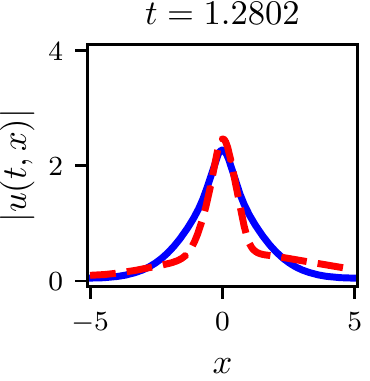}} \hfill
    \subfigure[FC-R]{\includegraphics[height=2.5cm, width=2.5cm]{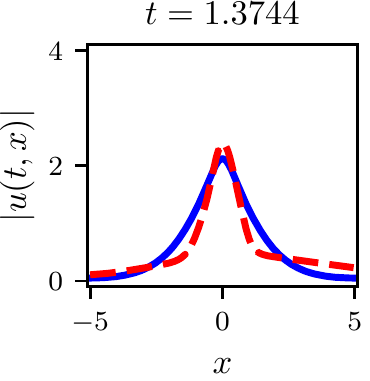}}\hfill
    \subfigure[FC-R]{\label{fig:reg_sch}\includegraphics[height=2.5cm, width=2.5cm]{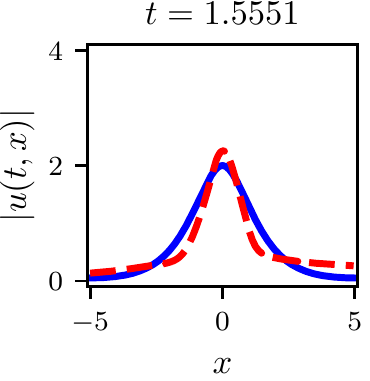}}
    \caption{We visualize the results by regression models. As shown, they are inferior to our PINN-D2 in Figs.~\ref{fig:ac} and~\ref{fig:sch}. Top two rows: the regression extrapolation snapshots for the Allen--Cahn equation. Bottom: the regression extrapolation snapshots for the nonlinear Schr\"{o}dinger equation.} 
    \label{fig:reg}
\end{figure}

\begin{table*}[!t]
\setlength{\tabcolsep}{2pt}
\centering
\scriptsize 
\caption{The extrapolation accuracy in terms of the relative errors in L2-norm, the Pearson correlation coefficient, and $R^2$ in various PDEs. Large (resp. small) values are preferred for $\uparrow$ (resp. $\downarrow$).}\label{tbl:reg}
\begin{tabular}{|c|c|c|c|c|c|c|c|c|c|c|c|c|c|c|c|c|}
\hline
\multirow{2}{*}{PDE} & \multicolumn{4}{c|}{L2-norm $(\downarrow)$} & \multicolumn{4}{c|}{Explained variance score $(\uparrow)$} & \multicolumn{4}{c|}{ Max error $(\downarrow)$} & \multicolumn{4}{c|}{ Mean absolute error $(\downarrow)$} \\ \cline{2-17}
          & FC & FC-R & PINN-D1 & PINN-D2 & FC & FC-R & PINN-D1 & PINN-D2 & FC & FC-R & PINN-D1 & PINN-D2 & FC & FC-R & PINN-D1 & PINN-D2 \\ \hline
Vis. Burgers & 0.352 & 0.301 & 0.112 & \textbf{0.092}  & 0.896 & 0.915 & 0.988 & \textbf{0.991} &
0.718 & 0.598 & 0.545 & \textbf{0.333} &
0.119 & 0.108 & 0.026 & \textbf{0.021} \\ \hline
Inv. Burgers & 0.114 & 0.133 & \textbf{0.083} & 0.090  & 0.060 & -0.181 & 0.454 & \textbf{0.621} &
3.245 & 3.301 & \textbf{1.534} & 2.036 &
0.255 & 0.332 & \textbf{0.277} & 0.315 \\ \hline
Allen--Cahn & 0.324 & 0.313 & 0.246 & \textbf{0.182} & 0.873 & 0.766 & 0.939 & \textbf{0.967} &
1.512 & 1.190 & 1.096 & \textbf{0.8366} & 
0.207 & 0.336 & 0.129 & \textbf{0.094} \\ \hline
Schr\"{o}dinger & 0.375 & 0.235  & 0.314 & \textbf{0.141} & -3.438 & -3.174 & -4.973 & \textbf{-3.257} & 4.078 & 4.3165 &  4.945 & \textbf{3.829} & 2.072 & 1.868 & \textbf{0.868} & 0.896 \\ \hline
\end{tabular}
\end{table*}

We train two regression models: one based on a series of fully connected (FC) layers and the other based on residual connections. In Table~\ref{tbl:reg}, they are denoted by FC and FC-R, respectively. We note that the neural network architecture of FC (resp. FC-R) is the same as that of PINN (resp. PINN-R) but they are trained in the supervised manner described earlier. We use the same set of hyperparameters for the number of layers, the dimensionality of hidden vector, and so on, and choose the best hyperparameter using the validation set. Note that this is exactly the same environment as the experiments shown in the main text.


Our proposed PINN-D1 and D2 outperform the regression-based models for all benchmark problems by large margins in Table~\ref{tbl:reg}. In Fig.~\ref{fig:reg}, we show the extrapolation results of the two regression models for the worst and the best performing cases in terms of human visual perception. For the AC equation (the top row in the figure), it is hard to say that they learned the physical dynamics. In particular, FC-R shows the worst extrapolation in Fig.~\ref{fig:reg_ac}. On the other hand, FC-R is successful for the NLS equation (the bottom row in the figure), whereas FC fails to extrapolate the both ends of the curve. Therefore, we can say that the regression-based approach shows unstable performance and is not ``physics-informed.''

\section{Conclusions}
In this work, we presented a novel training method, dynamic pulling method (DPM), for obtaining better performing physics-informed neural networks in extrapolation. The proposed DPM enables PINN to learn dynamics of the governing equations accurately. 
In the numerical experiments, we first demonstrated that the original PINN performs poorly on extrapolation tasks and empirically analyzed PINN in detail. Then, we demonstrated that the proposed DPM significantly outperforms PINN and its residual-block variant (up to 72\% in comparison with PINN and PINN-R) in various metrics. As an ablation study, we compared PINN-D1 and PINN-D2, where PINN-D2 overwhelms PINN-D1 in three benchmark problems. Finally, we explained how DPM behaves by illustrating example training loss curves. All codes and data will be released upon publication.

\clearpage
\section{Acknowledgements} \label{sec:ack}
This paper describes objective technical results and analysis. Any subjective
views or opinions that might be expressed in the paper do not necessarily
represent the views of the U.S. Department of Energy or the United States
Government. Sandia National Laboratories is a multimission laboratory managed
and operated by National Technology and Engineering Solutions of Sandia, a
wholly owned subsidiary of Honeywell International, for the U.S. Department of
Energy's National Nuclear Security Administration under contract DE-NA-0003525.
\bibliography{ref}
\bibliographystyle{aaai}

\clearpage
\appendix
\section{List of nonlinear PDEs}
Here we present a list of benchmark problems considered in this study. Following the original PINN paper, we consider 1D Burgers' equation, 1D nonlinear Schr\"{o}dinger equation, and Allen--Cahn equation. Here, $\sol_t \stackrel{\text{def}}{=} \frac{\partial u}{\partial t}$, $\sol_x \stackrel{\text{def}}{=} \frac{\partial u}{\partial x}$, and $\sol_{xx} \stackrel{\text{def}}{=} \frac{\partial^2 u}{\partial x^2}$.

\paragraph{1D viscous Burgers' equation.} As the first benchmark problem, we consider one-dimensional viscous Burgers' equation:
\begin{footnotesize}\begin{linenomath}\begin{equation}
    \governEqs(x,t) = \sol_t +\sol \sol_x - (0.01/pi) \sol_{xx}  =0,\; x\in[-1,1], \, t\in[0,1],
\end{equation}\end{linenomath}\end{footnotesize}
with the initial condition $\InitCond=-\sin(\pi x)$ and the boundary condition $\sol(t,-1)=\sol(t,1)=0, \forall t \in [0,1]$.

\paragraph{1D inviscid Burgers' equation.} Next, we consider one-dimensional inviscid Burgers' equation:
\begin{footnotesize}\begin{linenomath}\begin{equation}
    \governEqs(x,t) = \sol_t + \sol \sol_x - 0.02e^{0.015 x}=0,\; x \in [0,100],\, t \in [0, 35],
\end{equation}\end{linenomath}\end{footnotesize}
with the initial condition $\InitCond = 1,\forall x \in [0, 100],$ and the boundary condition $
\sol(0,t) = 4.25, \forall t \in [0, 35]$. This equation was not used in the original PINN paper but we added to test in more diverse environments.

\paragraph{1D nonlinear Schr\"{o}dinger (NLS) equation.} The third benchmark problem is one-dimensional nonlinear Schr\"{o}dinger equation, of which principal applications include the propagation of light in nonlinear optical fibers and planar waveguides:
\begin{footnotesize}\begin{linenomath}\begin{equation}
    \governEqs(x,t) = \sol_t - i0.5 \sol_{xx} - i |\sol|^2\sol,\; x \in [-5,5], \, t \in [0, \pi/2],
\end{equation}\end{linenomath}\end{footnotesize}
with the initial condition $\InitCond=2 \text{sech}(x), \forall x \in [-5,5]$, and the periodic boundary conditions $\sol(-5,t) = \sol(5,t)$ and $\sol_x(-5,t)=\sol_x(5,t), \forall t \in [0,\pi/2]$.

\paragraph{Allen--Cahn (AC) equation.} For the fourth benchmark problem, we consider a nonlinear reaction-diffusion problem, Allen--Cahn equation, which describes the process of phase separation in alloys:
\begin{footnotesize}\begin{linenomath}\begin{equation}
    \governEqs(x,t) = \sol_t - 0.0001 \sol_{xx} + 5\sol^3 - 5\sol=0, x \in [-1,1],\, t \in [0,1],
\end{equation}\end{linenomath}\end{footnotesize}
with the initial condition $\InitCond=x^2 \cos(\pi x), \forall x \in [-1,1]$, and the periodic boundary conditions $\sol(-1,t) = \sol(1,t)$ and $\sol_x(-1,t)=\sol_x(1,t), \forall t \in [0,1]$.


For computing the reference solutions of the 1D inviscid Burgers' equation, we consider the formulation shown in  \cite{white2003trajectory}: we use finite volume discretization with 256 cells for the spatial discretization and the implicit backward Euler scheme for the temporal discretization.
Following \cite{raissi2019physics}, the reference solutions of the other two equations are collected by using the Chebfun package \cite{driscoll2014chebfun}. For the NLS equation, a spectral Fourier discretization with 256 modes and a fourth-order explicit Runge--Kutta temporal integrator with time-step $\pi/2 \cdot 10^{-6}$ are considered. For the AC equation, a spectral Fourier discretization with 512 modes and a fourth-order explicit Runge--Kutta temporal integrator with time-step $10^{-5}$ and $10^{-6}$ are used, respectively.

\section{How to calculate the optimal gradient manipulation vector}

Finding the smallest $\bm{v}$ that meets Eq.~\eqref{eq:prob} can be formulated as follows:
\begin{linenomath}\begin{align}\begin{split}\label{eq:prob2}
\argmin_{\bm{v}}\quad& \|\bm{v}\|^2_2,\\
\textrm{subject to}\quad& \bm{v}\cdot \bm{g}^{(k)}_{L_f} = -\bm{g}^{(k)}_{L} \cdot \bm{g}^{(k)}_{L_f} + \delta.
\end{split}\end{align}\end{linenomath}

While one can directly derive many solutions after rearranging the above equality constraint w.r.t. $\bm{v}$, we explicitly solve the minimization problem to ensure the smallest solution among all feasible solutions satisfying the constraint. Since the objective is a convex function and the constraint is an affine function, its optimal solution can be analytically found.

\begin{theorem}
The optimal manipulation vector is $\bm{v}^* = \frac{-\bm{g}^{(k)}_{L} \cdot \bm{g}^{(k)}_{L_f} + \delta}{\| \bm{g}^{(k)}_{L_f} \|^2_2} \bm{g}^{(k)}_{L_f}$.
\end{theorem}
\begin{proof}
Recall that $\bm{g}^{(k)}_{L}$ and $\bm{g}^{(k)}_{L_f}$ are constant vectors when deciding $\bm{v}$ and therefore, the above problem~\eqref{eq:prob2} is a convex minimization with an affine constraint, which can be efficiently solved using the method of Langrage multiplier. We first construct the following Lagrangian $G$
\begin{linenomath}\begin{align}\begin{split}\label{eq:prob3}
G =& \|\bm{v}\|^2_2 + \lambda \big( \bm{v}\cdot \bm{g}^{(k)}_{L_f} + \bm{g}^{(k)}_{L} \cdot \bm{g}^{(k)}_{L_f} - \delta\big),
\end{split}\end{align}\end{linenomath}where $\lambda \in \mathbb{R}$ is a Lagrange multiplier.

Then, we can define the following equation system from the stationary condition which says that the derivative of $G$ becomes zero at the optimal solution:
\begin{linenomath}\begin{align}\begin{split}\label{eq:prob4}
\frac{\partial G}{\partial \bm{v}} =& 2\bm{v} + \lambda\bm{g}^{(k)}_{L_f}=\bm{0},\\
\frac{\partial G}{\partial \lambda} =& \bm{v}\cdot \bm{g}^{(k)}_{L_f} + \bm{g}^{(k)}_{L} \cdot \bm{g}^{(k)}_{L_f} - \delta=0.
\end{split}\end{align}\end{linenomath}

From Eq.~\eqref{eq:prob4}, the optimal solution $\lambda^*$ and $\bm{v}^*$ can be analytically found as follows:
\begin{linenomath}\begin{align}\begin{split}\label{eq:prob5}
\lambda^* =& \frac{2 \left(\bm{g}^{(k)}_{L} \cdot \bm{g}^{(k)}_{L_f} - \delta\right)}{\| \bm{g}^{(k)}_{L_f} \|^2_2},\\
\bm{v}^* =& \frac{-\bm{g}^{(k)}_{L} \cdot \bm{g}^{(k)}_{L_f} + \delta}{\| \bm{g}^{(k)}_{L_f} \|^2_2} \bm{g}^{(k)}_{L_f}.
\end{split}\end{align}\end{linenomath}
Note that this can be also viewed as a special case of Moore--Penrose inverse (pseudoinverse).
\end{proof}

\section{Additional Extrapolation Snapshots}
We show additional snapshots of the extrapolation results in Figs.~\ref{fig:burgersext}--\ref{fig:schext}. Note that some of them are from the validation set.

\begin{figure*}[!t]
    \centering
    \subfigure[PINN]{\includegraphics[width=0.16\textwidth]{images/Viscous_1/Burgers-Adam_0.8300_graph_6-40-0.0010-1-7950.pdf}} 
    \subfigure[PINN-R]{\includegraphics[width=0.16\textwidth]{images/Viscous_1/Burgers-Adam_ResNet_0.8300_graph_6-40-0.0125-1-2800.pdf}}
    \subfigure[PINN-D2]{\includegraphics[width=0.16\textwidth]{images/Viscous_1/Burgers-Adam_con_limit_ResNe_d2_0.8300_graph_8-20-0.0050-1.pdf}}
    \subfigure[PINN]{\includegraphics[width=0.16\textwidth]{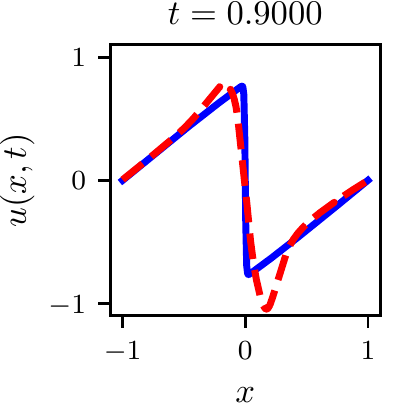}} 
    \subfigure[PINN-R]{\includegraphics[width=0.16\textwidth]{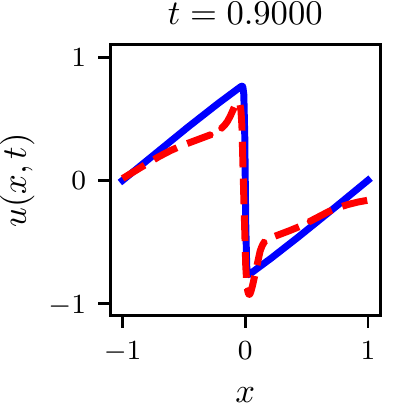}}
    \subfigure[PINN-D2]{\includegraphics[width=0.16\textwidth]{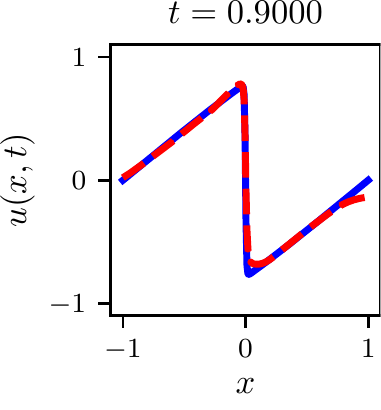}}
    \caption{The solution snapshots at various extrapolation $(x,t)$ pairs of the benchmark viscous Burgers' equation} 
    \label{fig:burgersext}

    \centering
    \subfigure[PINN]{\includegraphics[width=0.16\textwidth]{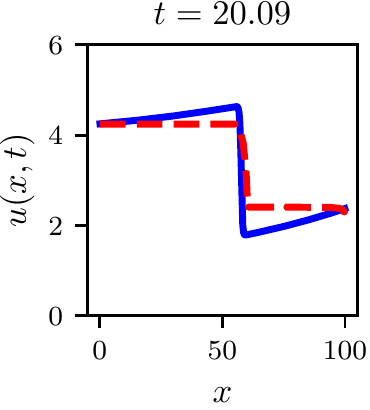}} 
    \subfigure[PINN-R]{\includegraphics[width=0.16\textwidth]{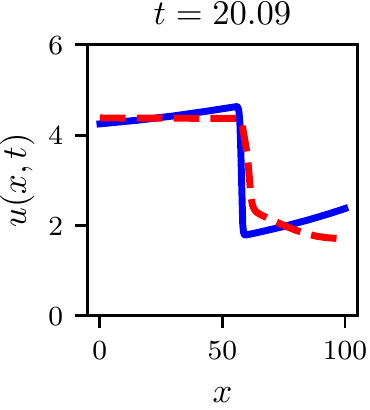}}
    \subfigure[PINN-D2]{\includegraphics[width=0.16\textwidth]{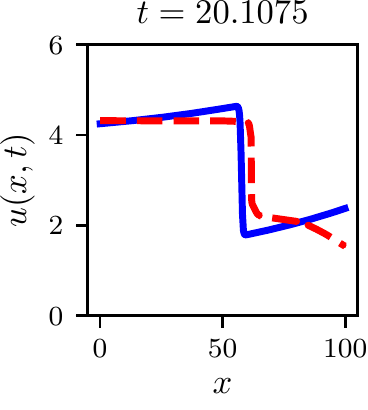}}
    \subfigure[PINN]{\includegraphics[width=0.16\textwidth]{images/Inviscid_1/Burgers-Adam_34.9125_graph_6-20-0.0125-1-11650.pdf}} 
    \subfigure[PINN-R]{\includegraphics[width=0.16\textwidth]{images/Inviscid_1/Burgers-Adam_ResNet_34.9125_graph_6-20-0.0125-1-6350.pdf}}
    \subfigure[PINN-D2]{\includegraphics[width=0.16\textwidth]{images/Inviscid_2/Burgers-Adam_con_limit_34.9125_graph_para.pdf}}
    \caption{The solution snapshots at various extrapolation $(x,t)$ pairs of the benchmark inviscid Burgers' equation} 
    \label{fig:burgers2ext}
    
    \centering
    \subfigure[PINN]{\includegraphics[width=0.16\textwidth]{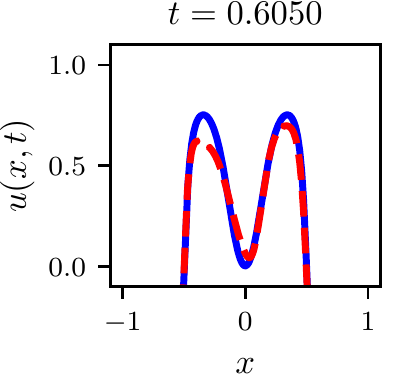}} 
    \subfigure[PINN-R]{\includegraphics[width=0.16\textwidth]{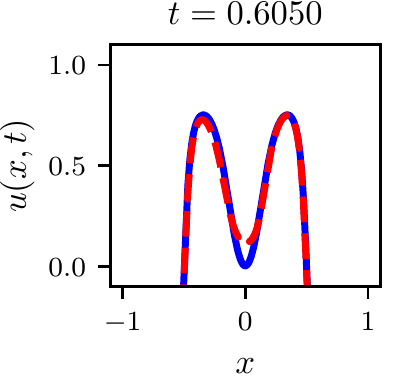}}
    \subfigure[PINN-D2]{\includegraphics[width=0.16\textwidth]{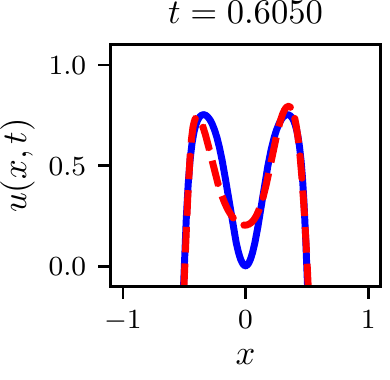}}
    \subfigure[PINN]{\includegraphics[width=0.16\textwidth]{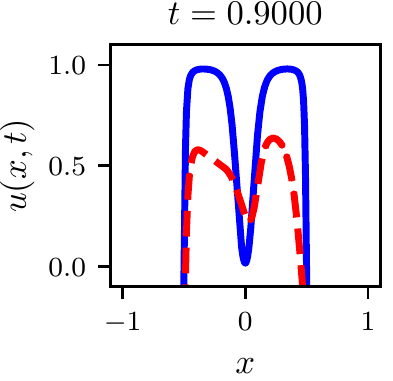}} 
    \subfigure[PINN-R]{\includegraphics[width=0.16\textwidth]{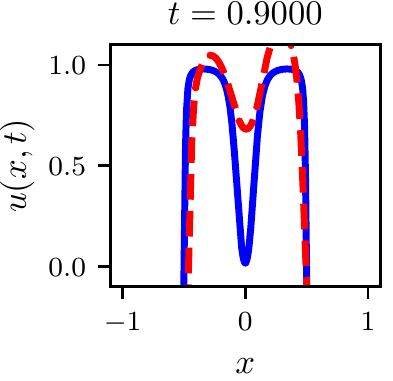}}
    \subfigure[PINN-D2]{\includegraphics[width=0.16\textwidth]{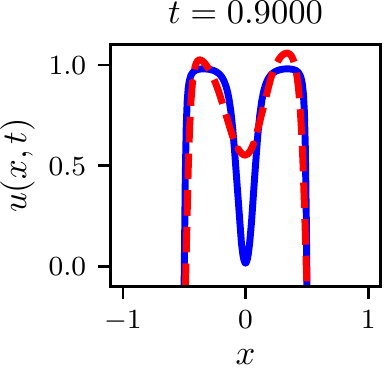}}
    \caption{The solution snapshots at various extrapolation $(x,t)$ pairs of the benchmark Allen--Cahn equation} 
    \label{fig:acext}
    
    \centering
    \subfigure[PINN]{\includegraphics[width=0.16\textwidth]{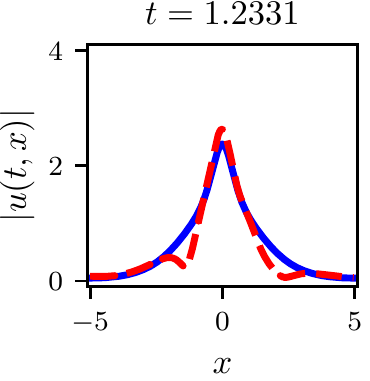}} 
    \subfigure[PINN-R]{\includegraphics[width=0.16\textwidth]{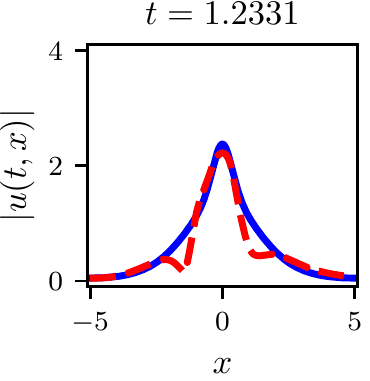}}
    \subfigure[PINN-D2]{\includegraphics[width=0.16\textwidth]{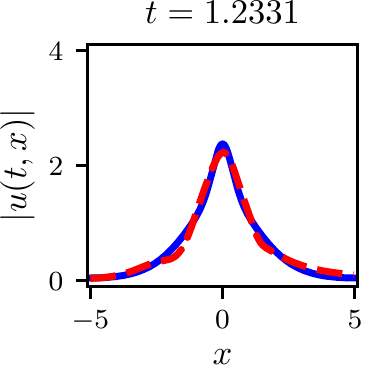}}
    \subfigure[PINN]{\includegraphics[width=0.16\textwidth]{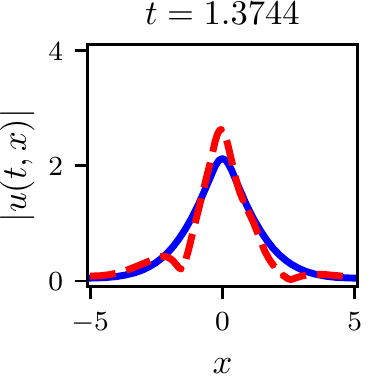}} 
    \subfigure[PINN-R]{\includegraphics[width=0.16\textwidth]{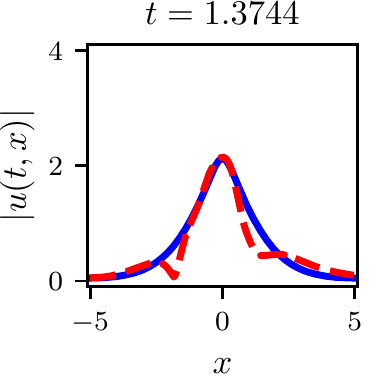}}
    \subfigure[PINN-D2]{\includegraphics[width=0.16\textwidth]{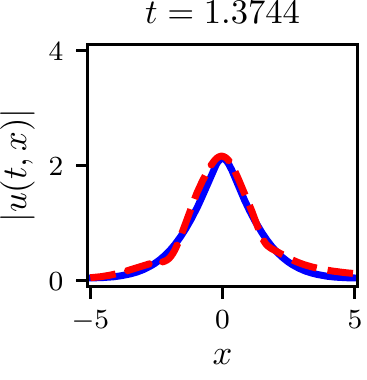}}
    \caption{The solution snapshots at various extrapolation $(x,t)$ pairs of the nonlinear Schr\"{o}dinger equation} 
    \label{fig:schext}
\end{figure*}

\section{Neural Network Architecture}
All of PINN, PINN-R, PINN-D1, and PINN-D2 use the following base layer:
\begin{align}
    \bm{h}_{i+1} = \sigma(\bm{W}^{\mathtt{T}}\bm{h}_i + \bm{b}),
\end{align}where $\sigma$ is the hyperbolic tangent, $\bm{h}_i$ is a hidden vector at $i$-th layer, and $\bm{W}, \bm{b}$ are trainable parameters. However, PINN-R, PINN-D1, and PINN-D2 are residual networks of the base layer. All hyperpameters are mentioned in the main paper.

\end{document}